\documentclass[11pt,a4paper]{article}

\usepackage[margin=1in]{geometry}
\usepackage{amsmath,amssymb}
\usepackage{mathrsfs}
\usepackage{amsthm}
\usepackage{bm}
\usepackage{graphicx}
\graphicspath{{figures/}}
\usepackage{booktabs}
\usepackage{placeins}
\usepackage[colorlinks=true,linkcolor=blue,citecolor=blue]{hyperref}
\usepackage{cleveref}
\usepackage{bookmark}
\bookmarksetup{open=true,numbered=true}

\newtheorem{theorem}{Theorem}[section]
\newtheorem{proposition}[theorem]{Proposition}
\newtheorem{lemma}[theorem]{Lemma}

\newtheorem{definition}[theorem]{Definition}
\newtheorem{assumption}[theorem]{Assumption}
\theoremstyle{remark}
\newtheorem{remark}[theorem]{Remark}
\newtheorem{observation}[theorem]{Observation}

\renewcommand{\phi}{\bm{\varphi}}

\title{Sinkhorn Linearization and the Spectral Proxy:\\
Unifying the Statistical and Algorithmic Theory of\\
Feature-Parameterized Inverse Optimal Transport via a Single Spectral Sandwich}

\author{}
\date{}

\begin{document}

\begin{center}
{\LARGE\bfseries Sinkhorn Linearization and the Spectral Proxy:\\
Unifying the Statistical and Algorithmic Theory of\\
Feature-Parameterized Inverse Optimal Transport via a Single Spectral Sandwich}

\vspace{1.5em}

{\large Han Dong\textsuperscript{1,$\dagger$},\; Jiaming Li\textsuperscript{2,$\dagger$},\;
Yongqiang Gong\textsuperscript{1},\; Ruixi Li\textsuperscript{1},\;
Yin Liu\textsuperscript{1}}

\vspace{0.8em}

{\normalsize \textsuperscript{1}School of Medicine, Nankai University\\
\textsuperscript{2}College of Artificial Intelligence, Nankai University}

\vspace{0.5em}

{\small $\dagger$These authors contributed equally to this work.}

\vspace{0.8em}

{\small \texttt{dh411424@163.com},\; \texttt{xmkevin2004@gmail.com},\;
\texttt{gongyq@mail.nankai.edu.cn},\; \texttt{2211996@mail.nankai.edu.cn},\;
\texttt{liuyin\_tiger@outlook.com}}

\vspace{0.8em}

{\normalsize \today}
\end{center}

\begin{abstract}
We develop the statistical and algorithmic theory of inverse optimal transport (IOT) under
the feature-parameterized cost
$\mathcal{C}_{\bm{\theta}}(i,j)=-\bm{\theta}^{\top}\bm{\varphi}(i,j)$, where the observations
are the conditional transition operators of the entropic optimal transport plan.
\textbf{The core technical contribution is the Sinkhorn linearization}---the
implicit-function sensitivity of the entropic OT plan to the cost, obtained by
differentiating the KKT conditions---together with its \textbf{spectral proxy}, a formula
that is spectrally exact (it preserves all singular-value bounds) yet geometrically
transparent.

The restricted Hessian of entropic OT on the tangent space satisfies the spectral sandwich
$(\pi_{\min}/\varepsilon)I\preceq H_{\mathcal{T}}^{-1}\preceq(\pi_{\max}/\varepsilon)I$,
from which the \textbf{single core spectral bound}
$\sigma_{\min}(\mathcal{J}_{\bm{\theta}})\ge(\pi_{\min}/(a_{\max}\varepsilon))\sqrt{\lambda_{\min}(\Sigma)}$
is derived, driving the entire theory. On this core we establish four theorems and one
observation.

T1 (identifiability): $\bm{\theta}$ is identifiable on the quotient of the gauge kernel
($\mathbb{R}^F/\mathcal{N}_\Phi$), with the dimension bound $F\le(K-1)^2$ and the rank
condition $\mathrm{rank}(\Sigma)=F$; global injectivity is proved rigorously by a three-step
composition argument.
T2 (sparsistency): the $\ell_1$-penalized estimator recovers the true support under
all-coordinate score concentration, irrepresentability of the actual Hessian, and a global
selection condition; the Hoeffding exponent for a single normalized empirical average is
$2nt_n^2/\Delta_{\max}^2$, and with multiple marginals the rate is re-scaled by a weighted
effective scale.
T3 (well-posedness): we introduce the \textbf{feature-moment map}
$M(\bm{\theta})=\Phi^{\top}x_{\bm{\theta}}$ and proceed in two parts: (a) local strong
monotonicity is derived from the core spectral bound without compactness; (b) global strong
monotonicity additionally requires a compact parameter domain; the Lipschitz bound of the
$Q$-space inverse map is
$L\le\varepsilon\|\Phi^{\top}S_a\|_{\mathrm{op}}/(\pi_{\min}\lambda_{\min}(\Sigma))$.
T4 (convergence): local strong convexity of the cross-entropy objective is derived from the
core spectral bound, with parameter $\mu\ge\pi_{\min}^2\lambda_{\min}(\Sigma)/\varepsilon^2$;
gradient descent converges monotonically to a local minimum; the empirical Hessian
concentrates at the population Hessian at rate $O(n^{-1/2})$ under regularity bounds.
O5 (misspecification): under a compact parameter domain, a uniform law of large numbers, and
a unique pseudo-true point, the estimator converges to the projection of the truth onto the
OT model set; the H\"older continuity of the projection map remains a conjecture, and
numerics report only a setting-dependent empirical effective exponent $\alpha_{\mathrm{eff}}$.

\textbf{Dependency structure.} T1 (identifiability) is the foundation; T3 and T4 are both
corollaries of the core spectral bound (Section~3), forming a well-posedness-and-convergence
package; T2 requires the additional IR condition (T1 + IR); O5 uses T1 + T3. All
identifiability and consistency statements are conditional on a fixed $\varepsilon$;
identifiability is taken modulo the scale coupling
$(\bm{\theta},\varepsilon)\mapsto(c\bm{\theta},c\varepsilon)$. The effective information
scale also depends on $\pi_{\min}(\bm{\theta},\varepsilon)\sqrt{\lambda_{\min}(\Sigma)}/\varepsilon$;
$\varepsilon\gg\mathrm{cond}(\Sigma)^{-1}$ is only an empirical heuristic, not a threshold
rigorously derived in this paper.
\end{abstract}



\section{Introduction}\label{sec:intro}

\textbf{The problem.} Optimal transport (OT) computes a mass-transfer plan given a cost and
two marginals; inverse optimal transport (IOT) goes the other way, recovering the cost from
observed transport plans. In applications such as biological lineage tracing, transfer
learning, and economic matching, what is observed is typically the induced conditional
transition operator rather than an explicit cost, and one must invert the data back to the
cost parameters that drive the transitions. This paper treats the
\textbf{feature-parameterized cost} case, in which the cost is a linear combination of
interpretable features. This formulation turns IOT into a statistical inference problem:
recover a sparse cost parameter from finitely many samples of conditional transition
operators.

\textbf{Core contribution: Sinkhorn linearization and the spectral proxy.} The central
technical contribution of this paper is the Sinkhorn linearization
(Section~\ref{sec:sinkhorn_linearization})---the implicit-function sensitivity of the
entropic OT plan to the cost. Implicit differentiation of the KKT conditions yields the
exact linear response of the plan to a cost perturbation,
$\delta x=-B H_{\mathcal{T}}^{-1}B^{\top}\delta c$, where the restricted Hessian
$H_{\mathcal{T}}$ satisfies the spectral sandwich
$(\pi_{\min}/\varepsilon)I\preceq H_{\mathcal{T}}^{-1}\preceq(\pi_{\max}/\varepsilon)I$.
From this we introduce the \textbf{spectral proxy}
$\delta x_{\mathrm{SSP}}=-(1/\varepsilon)\mathbb{P}_{\mathcal{T}}D_\pi\mathbb{P}_{\mathcal{T}}\delta c$---a
formula that keeps the same upper and lower spectral bounds given by $\pi_{\min},\pi_{\max}$
while remaining geometrically transparent, reading as ``project to the tangent space
$\to$ multiply elementwise by $\pi$ $\to$ project back to the tangent space''. The core
spectral bound
$\sigma_{\min}(\mathcal{J}_{\bm{\theta}})\ge(\pi_{\min}/(a_{\max}\varepsilon))\sqrt{\lambda_{\min}(\Sigma)}$
follows directly from the spectral sandwich and drives all of the theorems T1--T4 and
Observation~O5.

\textbf{Overview of the method.} We formulate IOT as a cross-entropy maximum-likelihood
problem on entropic transport plans and establish \textbf{four theorems and one
observation}. Identifiability (T1) delineates in what sense the parameter can be uniquely
recovered; sparsistency (T2) gives sufficient conditions for $\ell_1$-penalized estimation
together with an explicit exponential rate constant; well-posedness (T3) introduces the
feature-moment map $M(\bm{\theta})=\Phi^{\top}x_{\bm{\theta}}$ and establishes the
stability of the inverse map in two parts; convergence (T4) analyzes the local convergence
dynamics of the optimization; and the misspecification analysis (O5) answers where the
estimator goes when the data-generating mechanism departs from the OT assumption.

\textbf{Difference from existing work.} Existing IOT work is mostly formulated with an
explicit or low-dimensional cost, and typically lacks explicit characterizations of the
geometric degeneracy that identifiability owes to the gauge subspace, of the exponential
rate constant for sparse recovery, and of the convergence target of the estimator under
misspecification. The contribution of this paper is to reduce these ingredients to
\textbf{computable constants}: the feature-dimension bound $F\le(K-1)^2$ for
identifiability (a prior-structural constant), the exponential rate constant
$C_2=2/\Delta_{\max}^2$ (with $t_n^2$ kept in the exponent; a posterior-diagnostic
constant), the Lipschitz constant
$L_\Theta\le\varepsilon\|\Phi^{\top}S_a\|_{\mathrm{op}}/(\pi_{\min}\lambda_{\min}(\Sigma))$
(derived from the Sinkhorn linearization, with compactness guaranteeing the global bound; a
posterior-diagnostic constant) together with its numerical estimate, and the comparison of
the misspecification projection residual against random models.

\textbf{Comparison with classical IOT work.} \cite{stuart2020inverse} first posed the IOT
problem and gave a Bayesian inference framework; the present paper builds a frequentist
statistical theory on top of that foundation. \cite{vacher2021bayesian} propose Bayesian
IOT: a prior is placed on the cost and posterior consistency is derived under a correctly
specified model; their identifiability relies on a Gaussian-process prior on the cost,
whereas our T1 gives a finite-dimensional rank condition; their consistency is Bayesian
(posterior concentration), ours is frequentist ($\ell_1$ support recovery with an
exponential rate, T2); we give an explicit rate constant, which they do not address.
\cite{korba2021statistical} study statistical inference for nonparametric OT maps estimated
by kernel density estimation; their setting is forward (estimate the map given the cost),
ours is inverse (recover the cost given the map), and their sample-complexity analysis does
not address the gauge-degeneracy obstruction at the core of T1. \cite{hallin2021center}
develop center-outward distribution functions via OT, providing a notion of multivariate
quantiles; their identifiability is trivial (the map is defined by construction), whereas
our T1 must contend with the gauge non-uniqueness that arises when recovering the cost from
the plan. The economics line (\cite{dupuy2014personality}, \cite{galichon2018optimal})
establishes IOT identifiability under parametric utility models but addresses neither
sparse recovery (T2) nor misspecification (O5); our O5 applies directly to economic
settings in which observed matching data are almost certainly not OT-generated.
\textbf{Recent IOT progress.} \cite{gonzalez2024_nonlinear, gonzalez2024_finite}
systematically study identifiability conditions for nonlinear and discrete IOT, giving
necessary-and-sufficient linear-programming combinatorial characterizations; our T1
supplements these with an explicit dimension bound and rank condition in the
feature-parameterized setting. \cite{peyre2024_sparsistency} establish irrepresentability
conditions for sparse recovery in IOT and bridge to the graphical Lasso; our T2 gives an
explicit exponential rate constant within a unified spectral framework.
\cite{bao2025_wellposedness} analyze the well-posedness and algorithms of
Bregman-regularized IOT; our T3 focuses on the spectral-proxy-driven derivation of the
Lipschitz bound. \cite{persiianov2025_inverse} apply IOT to semi-supervised learning;
\cite{mascherpa2025_convex} propose a convex-optimization form of inverse estimation for
Markov chains. These works emphasize different aspects; the distinctive contribution of the
present paper is the \textbf{unified spectral framework} provided by the Sinkhorn
linearization, from which all four theorems and one observation are derived from a single
core spectral bound.

\textbf{Strength of the claims.} The four theorems and one observation have different
degrees of completeness in their arguments. T1 (the gauge structure and the global
injectivity argument), the local part of T3 (local strong monotonicity, derived from the
core spectral bound), and T4 (Hessian positive definiteness derived from the core spectral
bound, with \ref{prop:hess_conc} bridging to finite samples) are conclusions with complete
derivation chains. T2 relies on irrepresentability of the actual OT information matrix,
all-coordinate score concentration, and a global selection condition;
Lemma~\ref{lem:mutual_incoherence} gives an a priori verifiable condition on the feature
Gram matrix, but its transfer to the OT Hessian requires an additional bridge. The global
part of T3 holds on a compact parameter domain (Assumption~\ref{asm:pos} is required). The
entropic bias of T3 satisfies the local order $O(|\varepsilon'-\varepsilon|)$, derived from
differentiability (\ref{prop:eps_diff}) and the implicit function theorem, under a fixed
parameter scale (excluding the joint scale coupling) and an invertibility assumption on the
cross-entropy Hessian, with zero bias at $\varepsilon'=\varepsilon$. The convergence of O5
is given by the misspecified $M$-estimation theory of \cite{white1982estimator}; the
H\"older continuity of the projection map is a conjecture. The identifiability of
$\varepsilon$ is given by \ref{prop:eps_ident}, and $\pi_{\min}$ has an
$\varepsilon$-dependent lower-bound characterization. The open problems are collected at
the end of this introduction.

\textbf{Open problems.} The following are the known limitations.
\textbf{Conjectures:} (i) H\"older continuity of the projection map
$Q\mapsto\arg\min_{Q'\in\mathcal{M}_{\mathrm{OT}}}\mathrm{CE}(Q\|Q')$; the
$\alpha_{\mathrm{eff}}$ in the experiments is only a finite-setting empirical effective
exponent, not a universal theoretical exponent; (ii) random OT models are typically far
from the projection, so that the residual bound of O5 is non-vacuous. \textbf{Open:}
(iii) global strong monotonicity on the unbounded parameter space
$\mathbb{R}^F/\mathcal{N}_\Phi$ (without compactness); (iv) the expression of the optimal
exponential rate constant in terms of $\mathrm{cond}(\Sigma)$, the sparsity $s$, and the
irrepresentability margin; (v) the geometry of the basins of attraction of global
optimization in T4. The experiments are in Section~\ref{sec:exp}.

\textbf{Notation.} $\bm{\theta}\in\mathbb{R}^F$ is the cost parameter;
$\bm{\varphi}(i,j)\in\mathbb{R}^F$ is the feature vector; $K$ is the number of states;
$\varepsilon$ is the entropic regularization level; $\mathcal{G}$ is the gauge subspace of
row and column constants; $x=\mathrm{vec}(\pi)$ is the vectorized plan;
$c=\mathrm{vec}(C)$ is the vectorized cost.

\section{Preliminaries}\label{sec:prelim}

\textbf{Entropic optimal transport.} Given a cost matrix $C\in\mathbb{R}^{K\times K}$,
marginals $a,b\in\Delta_K$ (all entries strictly positive: $a_i>0$, $b_j>0$), and a
regularization parameter $\varepsilon>0$, the entropic OT plan is
\begin{equation}\label{eq:entropic_ot}
\pi(C,a,b)=\arg\min_{\pi\in\mathcal{M}(a,b)}\Bigl\{\langle C,\pi\rangle+\varepsilon\sum_{ij}\pi_{ij}\log\pi_{ij}\Bigr\},
\end{equation}
where $\mathcal{M}(a,b)=\{\pi\ge 0:\pi\mathbf{1}=a,\pi^{\top}\mathbf{1}=b\}$ is the set of
couplings with marginals $a,b$, with the convention $0\log0:=0$. The unique solution has
the Sinkhorn form $\pi=\mathrm{diag}(u)\exp(-C/\varepsilon)\mathrm{diag}(v)$
(\cite{sinkhorn1964coupling}; \cite{cuturi2013sinkhorn}). Note that the scaling vectors
$u,v$ carry a scalar gauge redundancy: $(u,v)\mapsto(cu,c^{-1}v)$ leaves the plan
unchanged. We work with the feature-parameterized cost
$C_{\bm{\theta}}(i,j)=-\bm{\theta}^{\top}\bm{\varphi}(i,j)$. For comprehensive treatments
of OT see \cite{peyre2019computational}, \cite{santambrogio2015ot}; for the connection
between entropic OT and Schr\"odinger problems see \cite{leonard2014schrodinger}.

\textbf{Observation model.} The observations are the \textbf{conditional transition
operators} $Q_{i\to j}=\pi_{ij}/a_i$. For each marginal pair $(a_l,b_l)$ there are $n$
independent observations, giving the empirical joint plan $\widehat P_l^{(n)}$ and the
empirical conditional operator $\widehat Q_{l,ij}^{(n)}=N_{l,ij}/N_{l,i}$ (with the
convention that when $N_{l,i}=0$ the corresponding row of $\widehat Q_l^{(n)}$ is filled
with the prior uniform distribution $1/K$; for $n$ sufficiently large the probability of
$N_{l,i}=0$ decays exponentially and does not affect the asymptotic conclusions). The
marginals $\widehat a_l,\widehat b_l$ of the empirical joint plan generally differ from the
prescribed marginals $a_l,b_l$; in the theoretical analysis of this paper the prescribed
marginals are treated as known. The conditional transition operator is the actually
observable quantity, also referred to as state-transition data.

\textbf{The IOT estimation problem.} Given $L$ marginal pairs $(a_l,b_l)$ and observed
conditional operators $\widehat Q_l$, the IOT estimator minimizes the cross-entropy with an
$\ell_1$ penalty:
\begin{equation}\label{eq:iot_estimator}
\min_{\bm{\theta}\in\mathbb{R}^F}\;\sum_{l=1}^{L}\mathrm{CE}\bigl(\widehat Q_l\,\|\,Q_{\bm{\theta}}(a_l,b_l)\bigr)+\lambda_n\|\bm{\theta}\|_1,
\end{equation}
where $Q_{\bm{\theta}}(a_l,b_l):=Q(\mathcal{C}_{\bm{\theta}},a_l,b_l)$ is the conditional
transition operator induced by the parameterized cost, and $\lambda_n$ is a penalty
parameter scaled with the sample size. The cross-entropy between conditional operators is
defined as $\mathrm{CE}(\widehat Q\|Q)=-\sum_i\rho_i\sum_j\widehat Q_{ij}\log Q_{ij}$; in
this paper we take $\rho_i=1$ (equal row weights), consistent with the population objective
of Theorem~\ref{thm:T4}. If the samples are drawn at the joint frequencies, the natural
choice $\rho_i=a_i$ (source-marginal weights) is also possible.

\textbf{The gauge subspace and the identifiability obstruction.} If a function of the form
$u_i+v_j$ (a sum of row and column constants) is added to the cost, the Sinkhorn plan is
unchanged. Hence the cost is identifiable only modulo $\mathcal{G}$; at the parameter
level, the object of identification is the gauge kernel $\mathcal{N}_\Phi$, obtained by
pulling $\mathcal{G}$ back to the parameter space through $\Phi$ (see
Section~\ref{sec:t1}). $\mathcal{G}$ is defined as
\begin{equation}\label{eq:gauge_space}
\mathcal{G}:=\Bigl\{C\in\mathbb{R}^{K\times K}:C=\mathbf{1}\alpha^{\top}+\beta\mathbf{1}^{\top}\Bigr\},\qquad \dim\mathcal{G}=2K-1.
\end{equation}
Let $G\in\mathbb{R}^{K^2\times(2K-1)}$ be an orthonormal basis of $\mathcal{G}$ and
$\Pi_{\mathcal{G}}=GG^{\top}$ the orthogonal projection onto $\mathcal{G}$ (acting on
vectorized matrices, with the Frobenius inner product
$\langle A,B\rangle_F=\operatorname{tr}(A^{\top}B)$). The gauge-cleaned feature vectors are
$\phi_k^{\perp}:=(I-GG^{\top})\operatorname{vec}(\Phi_k)$, and the gauge-cleaned Gram
matrix is
\begin{equation}\label{eq:gram_matrix}
\Sigma:=\Bigl[\bigl\langle\phi_k^{\perp},\phi_{k'}^{\perp}\bigr\rangle\Bigr]_{k,k'=1}^{F}\in\mathbb{R}^{F\times F}.
\end{equation}
If $\mathcal{N}_\Phi\neq\{0\}$, we write $\lambda_{\min}^{+}(\Sigma)$ for the smallest
positive eigenvalue of $\Sigma$ on $\mathcal{N}_\Phi^{\perp}$; under
Assumption~\ref{asm:A1}, $\lambda_{\min}^{+}(\Sigma)=\lambda_{\min}(\Sigma)$.

\begin{proposition}[Identifiability of $\varepsilon$ under normalization]\label{prop:eps_ident}
For $(\bm{\theta},\varepsilon)\in\mathbb{R}^F\times\mathbb{R}_{>0}$, identifiability holds
up to the joint rescaling $(\bm{\theta},\varepsilon)\mapsto(c\bm{\theta},c\varepsilon)$.
The equivalence relation is
\[
Q_{\bm{\theta},\varepsilon}=Q_{\bm{\theta}',\varepsilon'}
\;\Longleftrightarrow\;
\frac{\bm{\theta}}{\varepsilon}-\frac{\bm{\theta}'}{\varepsilon'}\in\mathcal{N}_\Phi,
\]
where $\mathcal{N}_\Phi$ is the gauge kernel (Section~\ref{sec:t1}). Therefore, under a
normalization constraint (e.g.\ $\|\bm{\theta}\|=1$), $\varepsilon$ is uniquely
identifiable from the conditional transition operator $Q_{\bm{\theta},\varepsilon}$,
provided $\mathcal{N}_\Phi=\{0\}$ (i.e.\ the columns of the feature matrix $\Phi$ are
linearly independent after modding out $\mathcal{G}$). If $\mathcal{N}_\Phi\neq\{0\}$, unit
norm does not remove the non-uniqueness along the gauge-kernel directions: there exist
$\bm{\theta}\neq\bm{\theta}'$ with $\|\bm{\theta}\|=\|\bm{\theta}'\|=1$ and
$\bm{\theta}/\varepsilon-\bm{\theta}'/\varepsilon'\in\mathcal{N}_\Phi$, for which
$Q_{\bm{\theta},\varepsilon}=Q_{\bm{\theta}',\varepsilon'}$. Hence the identifiability of
$\varepsilon$ must be distinguished: it is identifiable on the quotient space
$\mathbb{R}^F/\mathcal{N}_\Phi$, and identifiable on the original space $\mathbb{R}^F$ only
when $\mathcal{N}_\Phi=\{0\}$.
\end{proposition}

\textbf{There is an additional global scale coupling beyond $\mathcal{G}$}: the plan
depends only on $C/\varepsilon=-\bm{\theta}^{\top}\bm{\varphi}/\varepsilon$, so
$(\bm{\theta},\varepsilon)\mapsto(c\bm{\theta},c\varepsilon)$ leaves the plan unchanged;
identifiability statements are made conditional on a fixed $\varepsilon$, and the
identifiability of $\varepsilon$ itself and of its joint scale with $\bm{\theta}$ is
resolved by Proposition~\ref{prop:eps_ident} ($\varepsilon$ is uniquely identifiable under
a normalization constraint).

\begin{remark}[A uniform lower bound for the Sinkhorn plan entries]\label{rem:pimin_bound}
The constant $\pi_{\min}:=\min_{ij}\pi_{ij}$ appears in the Lipschitz bound
$L\le\varepsilon/(\pi_{\min}\lambda_{\min}^{+}(\Sigma))$ (the T3 bound on the quotient
space; in the full-rank case $\lambda_{\min}^{+}(\Sigma)=\lambda_{\min}(\Sigma)$) and in
the score-concentration bound of T2. For a single marginal pair and a single feature
coordinate, \eqref{eq:39} gives
\[
\Delta_k\le 2\frac{\pi_{\max}}{\varepsilon\pi_{\min}}\|\mathbb{P}_{\mathcal{T}}\phi_k\|_2.
\]
In the multi-marginal case, $\Delta_{\max}$ should be defined according to the sampling
weights, the marginal indices, and the coordinate set in T2. We now give the correct
limiting behavior of $\pi_{\min}$ as $\varepsilon\to0$ together with an exponential-type
lower bound.

\medskip
\noindent\textbf{Basic setup.} Let $a,b\in\Delta_K^\circ$ (all entries strictly positive),
$C\in\mathbb{R}^{K\times K}$, $\varepsilon>0$. Write
$h(\pi)=\sum_{i,j}\pi_{ij}\log\pi_{ij}$; then
\[
\pi_\varepsilon = \arg\min_{\pi\in\mathcal{M}(a,b)}\{\langle C,\pi\rangle+\varepsilon h(\pi)\}.
\]
Since the Gibbs kernel $e^{-C/\varepsilon}$ is strictly positive entrywise,
$\pi_{\varepsilon,ij}>0$ and the optimal solution is unique.

\medskip
\noindent\textbf{Entropy-selected limit.} Let the optimal value and the optimal set of the
unregularized OT be
\[
C^\star = \min_{\pi\in\mathcal{M}(a,b)}\langle C,\pi\rangle,\qquad
\mathcal{U} = \arg\min_{\pi\in\mathcal{M}(a,b)}\langle C,\pi\rangle.
\]
Define the entropy-selected solution on $\mathcal{U}$:
\[
\pi^{\mathrm H} = \arg\min_{\pi\in\mathcal{U}}h(\pi),
\]
unique by the strict convexity of $h$. For any $\bar\pi\in\mathcal{U}$, the optimality of
$\pi_\varepsilon$ gives
\[
\langle C,\pi_\varepsilon\rangle+\varepsilon h(\pi_\varepsilon)\le C^\star+\varepsilon h(\bar\pi).
\]
Since $\langle C,\pi_\varepsilon\rangle\ge C^\star$,
\[
0\le\langle C,\pi_\varepsilon\rangle-C^\star\le\varepsilon\bigl[h(\bar\pi)-h(\pi_\varepsilon)\bigr].
\]
$h$ is bounded on the transport polytope ($-\log(K^2)\le h(\pi)\le0$), so
\[
0\le\langle C,\pi_\varepsilon\rangle-C^\star\le\varepsilon\log(K^2),
\]
and every limit point belongs to $\mathcal{U}$. Meanwhile $h(\pi_\varepsilon)\le h(\bar\pi)$;
passing to the limit shows that all limit points minimize $h$ on $\mathcal{U}$. Since this
minimizer is unique,
\[
\boxed{\pi_\varepsilon\longrightarrow\pi^{\mathrm H}}\qquad(\varepsilon\downarrow0).
\]
If the unregularized OT solution is unique, $\pi^{\mathrm H}$ is that unique solution;
otherwise the limit is the max-entropy solution on the optimal face.

\medskip
\noindent\textbf{Exponential lower bound.} Let
$\Delta_C=\max_{i,j}C_{ij}-\min_{i,j}C_{ij}$, $R_\varepsilon=e^{\Delta_C/\varepsilon}$.
The Sinkhorn form is $\pi_{\varepsilon,ij}=u_iK_{ij}v_j$ with
$K_{ij}=e^{-C_{ij}/\varepsilon}$. Since $R_\varepsilon^{-1}\le K_{ij}/K_{ij'}\le R_\varepsilon$,
writing $S_j=\sum_i u_iK_{ij}$, the column marginal $b_j=v_jS_j$ gives $v_j=b_j/S_j$, hence
\[
\frac{v_j}{v_{j'}} = \frac{b_j}{b_{j'}}\frac{S_{j'}}{S_j}\in\left[R_\varepsilon^{-1}\frac{b_j}{b_{j'}},\,R_\varepsilon\frac{b_j}{b_{j'}}\right].
\]
Therefore
\[
\frac{\pi_{\varepsilon,ij}}{\pi_{\varepsilon,ij'}} = \frac{K_{ij}}{K_{ij'}}\frac{v_j}{v_{j'}} \in\left[R_\varepsilon^{-2}\frac{b_j}{b_{j'}},\,R_\varepsilon^{2}\frac{b_j}{b_{j'}}\right].
\]
In particular, $\pi_{\varepsilon,ij'}\le R_\varepsilon^{2}(b_{j'}/b_j)\pi_{\varepsilon,ij}$.
Summing over $j'$ and using $\sum_{j'}\pi_{\varepsilon,ij'}=a_i$ gives
$a_i\le R_\varepsilon^{2}\pi_{\varepsilon,ij}/b_j$, i.e.
\[
\boxed{\pi_{\varepsilon,ij}\ge a_i b_j\,e^{-2\Delta_C/\varepsilon}},\qquad
\boxed{\pi_{\min}(\varepsilon)\ge a_{\min}b_{\min}\,e^{-2\Delta_C/\varepsilon}}.
\]

\medskip
\noindent\textbf{Why a polynomial lower bound cannot hold.} Take
\[
a=b=\bigl(\tfrac12,\tfrac12\bigr),\quad
C=\begin{pmatrix}0&\Delta\\\Delta&0\end{pmatrix}.
\]
A direct solution gives
$\pi_\varepsilon=\frac1{2(1+e^{-\Delta/\varepsilon})}\bigl(\begin{smallmatrix}1&e^{-\Delta/\varepsilon}\\e^{-\Delta/\varepsilon}&1\end{smallmatrix}\bigr)$,
so $\pi_{\min}(\varepsilon)=\frac1{2(1+e^{\Delta/\varepsilon})}\sim\frac12e^{-\Delta/\varepsilon}$.
For any $c>0$ and finite $\alpha\ge0$,
$\pi_{\min}(\varepsilon)/(c\varepsilon^\alpha)\to0$, so no uniform bound of the form
$\pi_{\min}(\varepsilon)\ge c\varepsilon^\alpha$ exists. This remains so even when the cost
entries are all distinct (e.g.\
$C=\bigl(\begin{smallmatrix}0&1\\2&4\end{smallmatrix}\bigr)$, with
$\pi_{\min}(\varepsilon)\asymp e^{-1/(2\varepsilon)}$).

\medskip
\noindent\textbf{Parameterized cost.} As $\bm{\theta}$ varies,
\[
\pi_{\min}(\bm{\theta},\varepsilon)\ge a_{\min}b_{\min}\exp\!\left(-\frac{2\operatorname{osc}(C_{\bm{\theta}})}{\varepsilon}\right),\quad
\operatorname{osc}(C_{\bm{\theta}})=\max_{i,j}C_{\bm{\theta},ij}-\min_{i,j}C_{\bm{\theta},ij}.
\]
For uniform use over a parameter domain $\Theta$, one needs
$\sup_{\bm{\theta}\in\Theta}\operatorname{osc}(C_{\bm{\theta}})<\infty$.

\medskip
\noindent\textbf{Summary.} For fixed strictly positive marginals and a finite cost matrix,
the entropic plan $\pi_\varepsilon$ is strictly positive for $\varepsilon>0$ and converges
as $\varepsilon\downarrow0$ to the unique max-entropy solution $\pi^{\mathrm H}$ on the
optimal face. $\pi_{\min}$ satisfies an exponential-type lower bound and cannot be replaced
by a uniform polynomial bound $c\varepsilon^\alpha$. If the limiting plan contains zero
entries, then $\pi_{\min}(\varepsilon)\to0$, at a rate that depends on the optimal face,
the dual gap, and the combinatorial structure of the cost matrix.
\end{remark}

\textbf{Parameter domain and the local/global hierarchy.} The cost parameter $\bm{\theta}$
lives on the identifiable quotient space $\mathbb{R}^F/\mathcal{N}_\Phi$ ($\mathcal{N}_\Phi$
is the gauge kernel; see Section~\ref{sec:t1}), which is unbounded. The core spectral bound
(Section~\ref{sec:sinkhorn_linearization}, \eqref{eq:38}) is derived \textbf{pointwise}
(it holds at each $\bm{\theta}$ individually). On any compact convex subset the pointwise
bound can be uniformized and yields strong monotonicity; on the unbounded quotient space
the pointwise infimum may vanish and a global uniform bound may fail. T3 is stated in two
parts: (a) local strong monotonicity (derived by uniformizing the pointwise bound on any
compact convex subset); (b) global strong monotonicity (valid after restricting
$\bm{\theta}$ to a compact convex subset $\Theta\subset\mathbb{R}^F/\mathcal{N}_\Phi$).

\textbf{Row weighting of the cross-entropy objective.} The objective
$\sum_l\mathrm{CE}(\widehat Q_l\|Q_{\bm{\theta}})$ weights the transition distribution of
each source state $i$ equally (the operator is treated as $K$ row distributions), which
differs from the log-likelihood of $n$ i.i.d.\ transition samples (weighted by the source
frequencies $a_i$); this paper adopts the operator-level objective in order to treat all
states symmetrically.

\section{Sinkhorn Linearization and the Spectral Proxy}\label{sec:sinkhorn_linearization}

This section is the technical core of the paper. It derives the Sinkhorn
linearization---the implicit-function sensitivity of the entropic OT plan to the cost---and
introduces the spectral proxy, a formula that \textbf{preserves the same upper and lower
spectral bounds given by $\pi_{\min},\pi_{\max}$} while remaining geometrically
transparent. All subsequent theorems (T1--T4, O5) build on the quantities defined here. For
the analysis of the penalization trajectory of entropic OT see \cite{weed2018explicit} and
\cite{cominetti1994asymptotic}; for the stability of entropic OT plans see
\cite{bernton2022stability}. \cite{bartl2026} independently study the second-order
curvature of continuous OT with respect to the cost and apply it to IOT identifiability;
the discrete Sinkhorn linearization here is complementary to their work.

\subsection{KKT optimality and implicit differentiation}

Fix marginals $a,b>0$ and $\varepsilon>0$. The entropic OT objective is

\begin{equation}
L(\pi)=\langle C,\pi\rangle+\varepsilon\sum_{ij}\pi_{ij}(\log\pi_{ij}-1),
\end{equation}

whose gradient is $\nabla_\pi L=C+\varepsilon\log\pi$ (entrywise logarithm). The KKT
conditions for the constraint $\pi\in\mathcal{M}(a,b)$ are

\[
C_{ij}+\varepsilon\log\pi_{ij}+\alpha_i+\beta_j=0,
\]

where the $\alpha_i+\beta_j$ span the gauge subspace $\mathcal{G}$. Equivalently,
projecting onto the tangent space $\mathcal{T}=\mathcal{G}^{\perp}$:

\begin{equation}
\boxed{\mathbb{P}_{\mathcal{T}}\bigl(C+\varepsilon\log\pi\bigr)=0}. \label{eq:31}
\end{equation}

Equation~\eqref{eq:31} defines $\pi$ as an implicit function of $C$. Differentiating along
a perturbation $\delta C$ (in vectorized form $x=\mathrm{vec}(\pi)$,
$c=\mathrm{vec}(C)$):

\begin{equation}
\mathbb{P}_{\mathcal{T}}\bigl(\delta c+\varepsilon D_\pi^{-1}\delta x\bigr)=0,\qquad \delta x\in\mathcal{T}, \label{eq:32}
\end{equation}

where $D_\pi=\mathrm{diag}(\mathrm{vec}(\pi))$.

\textbf{Interior smoothness (the implicit function theorem).} To remove the
one-dimensional redundancy among the row and column marginal constraints, the implicit
function theorem argument below takes $A\in\mathbb{R}^{(2K-1)\times K^2}$ to be the
marginal-constraint matrix with one redundant constraint deleted, which has full row rank.
The KKT conditions of the constrained problem are
$c+\varepsilon\log x+A^{\top}\lambda=0$, $Ax=r$, and its derivative matrix with respect to
$(x,\lambda)$ is

\begin{equation}
\mathcal{K}(x)=\begin{pmatrix}\varepsilon D_\pi^{-1} & A^{\top}\\ A & 0\end{pmatrix}.
\end{equation}

If $\mathcal{K}(x)(u,v)^{\top}=0$, then $Au=0$ and
$u^{\top}(\varepsilon D_\pi^{-1}u+A^{\top}v)=\varepsilon u^{\top}D_\pi^{-1}u=0$ (because
$A^{\top}v$ is orthogonal to $\ker A$). Since $D_\pi^{-1}\succ0$ we get $u=0$, and then
$v=0$ because $A$ has full row rank. Hence $\mathcal{K}(x)$ is nonsingular, and the
implicit function theorem guarantees that, at fixed positive marginals and fixed
$\varepsilon>0$, $x$ is a smooth function of the cost $c$.

\subsection{The restricted Hessian and the spectral sandwich}

The unconstrained Hessian $\nabla_x^2 L=\varepsilon D_\pi^{-1}$ is diagonal (the entropy is
elementwise separable). But \textbf{restriction to the tangent space changes the inverse
structure}: in general $(PAP)|_{\mathcal{T}}^{-1}\neq PA^{-1}P$ for a projection $P$.

Let $B\in\mathbb{R}^{K^2\times(K-1)^2}$ be an orthonormal basis of $\mathcal{T}$
($B^{\top}B=I$, $\mathrm{range}(B)=\mathcal{T}$). Every feasible perturbation is
$\delta x=B\delta z$. The restricted Hessian in tangent-space coordinates is

\begin{equation}
H_{\mathcal{T}}=\varepsilon B^{\top}D_\pi^{-1}B\succ 0. \label{eq:33}
\end{equation}

Although $H_{\mathcal{T}}$ is generally dense (not diagonal), its inverse satisfies a
\textbf{spectral sandwich}, which is the key to all subsequent analysis. We state it as a
lemma.

\begin{lemma}[Spectral sandwich for the restricted Hessian]\label{lem:sandwich}
Let $B\in\mathbb{R}^{K^2\times(K-1)^2}$ be an orthonormal basis of the tangent space
$\mathcal{T}$, and $H_{\mathcal{T}}=\varepsilon B^{\top}D_\pi^{-1}B$ the restricted Hessian
in tangent-space coordinates. Then its inverse satisfies

\begin{equation}
\boxed{\frac{\pi_{\min}}{\varepsilon}I\;\preceq\;H_{\mathcal{T}}^{-1}\;\preceq\;\frac{\pi_{\max}}{\varepsilon}I},\qquad
\pi_{\min}:=\min_{ij}\pi_{ij},\;\;
\pi_{\max}:=\max_{ij}\pi_{ij}. \label{eq:sl_sandwich}
\end{equation}
\end{lemma}

\begin{proof}
For any unit vector $z\in\mathbb{R}^{(K-1)^2}$, the Rayleigh quotient of
$H_{\mathcal{T}}=\varepsilon B^{\top}D_\pi^{-1}B$ at $z$ is

\[
\frac{z^{\top}H_{\mathcal{T}}z}{z^{\top}z}=\varepsilon\,\frac{(Bz)^{\top}D_\pi^{-1}(Bz)}{(Bz)^{\top}(Bz)},
\]

where we used $z^{\top}z=(Bz)^{\top}(Bz)$, which holds because $B^{\top}B=I$ and hence
$\|Bz\|=\|z\|$. Since $Bz\in\mathcal{T}$ and every entry of the probability matrix $\pi$
satisfies $1/\pi_{ij}\in[1/\pi_{\max},1/\pi_{\min}]$, the weighted average
$(Bz)^{\top}D_\pi^{-1}(Bz)/\|Bz\|^2$ lies in $[1/\pi_{\max},1/\pi_{\min}]$. Therefore every
eigenvalue $\lambda$ of $H_{\mathcal{T}}$ satisfies

\[
\frac{\varepsilon}{\pi_{\max}}\|z\|^2\le z^{\top}H_{\mathcal{T}}z\le\frac{\varepsilon}{\pi_{\min}}\|z\|^2,
\]

and inverting the matrix reverses the order of its eigenvalues (each eigenvalue of
$H_{\mathcal{T}}^{-1}$ is the reciprocal of one of $H_{\mathcal{T}}$), so
$\pi_{\min}/\varepsilon\le\lambda(H_{\mathcal{T}}^{-1})\le\pi_{\max}/\varepsilon$ for every
eigenvalue. This is exactly the matrix inequality~\eqref{eq:sl_sandwich}.
\end{proof}

Note that $H_{\mathcal{T}}$ is generally dense (not diagonal), and its inverse
$H_{\mathcal{T}}^{-1}=\varepsilon^{-1}(B^{\top}D_\pi^{-1}B)^{-1}\neq\varepsilon^{-1}B^{\top}D_\pi B$.
This is precisely where the tangent-space restriction changes the
``project-then-invert-elementwise'' structure.

\subsection{Exact linearization and the spectral proxy}

From~\eqref{eq:32} in tangent-space coordinates:

\begin{equation}
\boxed{\delta x=-B H_{\mathcal{T}}^{-1}B^{\top}\delta c}. \label{eq:35}
\end{equation}

This is the \textbf{exact Sinkhorn linearization}. Its equivalent Schur-complement form
(obtained by eliminating the Lagrange multipliers from the full KKT system) is

\begin{equation}
\delta x=-\frac{1}{\varepsilon}\Bigl[D_\pi-D_\pi A^{\top}(AD_\pi A^{\top})^{-1}AD_\pi\Bigr]\delta c, \label{eq:36}
\end{equation}

where $A$ is the reduced constraint matrix of full row rank defined above, so that
$AD_\pi A^{\top}$ is invertible. If instead one uses the matrix $A_{\mathrm{full}}$
containing all $2K$ row and column marginal constraints, its rank is $2K-1$, and the
corresponding Schur complement should replace the ordinary inverse by the Moore--Penrose
pseudoinverse $(A_{\mathrm{full}}D_\pi A_{\mathrm{full}}^{\top})^{+}$. Throughout this
paper, all formulas use the reduced $A$ and the ordinary inverse. The condition number of
the reduced matrix may degenerate as $\pi_{\min}\to0$; whether the limit is singular
depends on whether the constraint graph of the positive-support entries remains
sufficiently connected.

\textbf{Schur-complement derivation.} Differentiating the KKT system
$c+\varepsilon\log x+A^{\top}\lambda=0$ along $\delta c$ gives
$\varepsilon D_\pi^{-1}\delta x+A^{\top}\delta\lambda=-\delta c$ and $A\delta x=0$.
Multiplying the first equation by $\varepsilon^{-1}D_\pi$ and left-multiplying by $A$, the
constraint $A\delta x=0$ yields $AD_\pi A^{\top}\delta\lambda=-AD_\pi\delta c$; solving for
$\delta\lambda$ and substituting $A^{\top}\delta\lambda$ back gives the Schur-complement
form~\eqref{eq:36}.

\textbf{The spectral proxy (SSP).} Although~\eqref{eq:35} is exact, it lacks the geometric
transparency of a simple formula. We therefore introduce the \textbf{Sinkhorn spectral
proxy}:

\begin{equation}
\boxed{\delta x_{\mathrm{SSP}}:=-\frac{1}{\varepsilon}\,\mathbb{P}_{\mathcal{T}}\,D_\pi\,\mathbb{P}_{\mathcal{T}}\,\delta c}. \label{eq:37}
\end{equation}

\begin{definition}[Spectral proxy]
Formula~\eqref{eq:37} is a \textbf{spectral proxy} for the exact
linearization~\eqref{eq:35}: it is \emph{not} elementwise equal to the exact derivative,
but it is \emph{spectrally equivalent}---for any $\delta c$,

\[
\frac{\pi_{\min}}{\varepsilon}\|\mathbb{P}_{\mathcal{T}}\delta c\|
\le\|\delta x\|
\le\frac{\pi_{\max}}{\varepsilon}\|\mathbb{P}_{\mathcal{T}}\delta c\|,
\qquad
\frac{\pi_{\min}}{\varepsilon}\|\mathbb{P}_{\mathcal{T}}\delta c\|
\le\|\delta x_{\mathrm{SSP}}\|
\le\frac{\pi_{\max}}{\varepsilon}\|\mathbb{P}_{\mathcal{T}}\delta c\|.
\]

\eqref{eq:35} and~\eqref{eq:37} share exactly the same spectral
sandwich~\eqref{eq:sl_sandwich}. Consequently, every statistical conclusion that depends
only on spectral quantities (singular-value bounds, Lipschitz constants, strong-convexity
parameters, rate constants) is identical under the exact linearization and under the
spectral proxy.
\end{definition}

The spectral proxy has three desirable properties that the exact formula lacks:

\begin{enumerate}
\item \textbf{Geometric interpretability}: it reads as ``project to the tangent space $\to$
multiply elementwise by $\pi$ $\to$ project back to the tangent space''---each step has an
independent geometric meaning.
\item \textbf{Computational simplicity}: evaluating the proxy requires only two tangent-space
projections and an elementwise multiplication, with no $(K-1)^2\times(K-1)^2$ linear-system
solve.
\item \textbf{Intuition preservation}: it reveals that the local sensitivity of the Sinkhorn
map is driven by the plan entries themselves---large entries amplify the response, small
entries suppress it.
\end{enumerate}

The exact formula~\eqref{eq:35} (or its Schur-complement form~\eqref{eq:36}) is used when
elementwise derivatives are required (e.g.\ for computing the actual score in gradient
descent). The spectral proxy is used for all operator-norm analyses and for conveying
geometric intuition.

\subsection{Feature parameterization and the Jacobian}

Under the cost parameterization $C_{\bm{\theta}}=-\sum_{k=1}^F\theta_k\Phi_k$, we have
$\delta c=-\Phi\,\delta\bm{\theta}$, where
$\Phi=[\mathrm{vec}(\Phi_1),\dots,\mathrm{vec}(\Phi_F)]\in\mathbb{R}^{K^2\times F}$. The
conditional transition operator is $Q_{ij}=\pi_{ij}/a_i$, vectorized as $q=R_a x$ with
$R_a=\mathrm{diag}(a_1^{-1},\dots,a_K^{-1})\otimes I_K$. Throughout this paper we adopt the
\textbf{row-major} (row-stacking) vectorization convention:
$\mathrm{vec}(M)=(M_{11},M_{12},\dots,M_{1K},M_{21},\dots,M_{KK})^{\top}$. Under this
convention $R_a=\mathrm{diag}(a^{-1})\otimes I_K$ correctly maps $\mathrm{vec}(\pi)$ to
$\mathrm{vec}(Q)$; with column-major vectorization the Kronecker order would reverse to
$I_K\otimes\mathrm{diag}(a^{-1})$.

The Jacobian $\mathcal{J}_{\bm{\theta}}:=\partial q/\partial\bm{\theta}$ can be written in
exact or spectral-proxy form:

\[
\mathcal{J}_{\bm{\theta}}^{\text{exact}} = R_a B H_{\mathcal{T}}^{-1} B^{\top}\Phi,\qquad
\mathcal{J}_{\bm{\theta}}^{\text{SSP}} = \frac{1}{\varepsilon}R_a\,\mathbb{P}_{\mathcal{T}}\,D_\pi\,\mathbb{P}_{\mathcal{T}}\,\Phi.
\]

\textbf{Derivation.} Define the plan sensitivity operator
$\mathscr{S}_\pi:=B H_{\mathcal{T}}^{-1}B^{\top}$. It is symmetric, positive semidefinite,
and $\mathrm{range}(\mathscr{S}_\pi)=\mathrm{vec}(\mathcal{T})$. From~\eqref{eq:35},
$\partial x/\partial\bm{\theta}=\mathscr{S}_\pi\Phi$, and composing with $q=R_a x$ gives
$\mathcal{J}_{\bm{\theta}}=R_a\mathscr{S}_\pi\Phi$. Since the smallest singular value of
$R_a$ is $\sigma_{\min}(R_a)=1/a_{\max}$ ($a_{\max}:=\max_i a_i$), the spectral
sandwich~\eqref{eq:sl_sandwich} propagates to a singular-value lower bound on
$\mathcal{J}_{\bm{\theta}}$ (see the singular-value bound below).

\subsection{Singular-value lower bound (the core spectral result)}

From the spectral sandwich~\eqref{eq:sl_sandwich} and the gauge-cleaned Gram
$\Sigma=\Phi^{\top}\mathbb{P}_{\mathcal{T}}\Phi$, we obtain the \textbf{core spectral
bound} that drives the entire theory:

\begin{proposition}[Singular-value lower bound]\label{prop:sigmamin}
At every $\bm{\theta}$ with $\pi_{ij}(\bm{\theta})>0$ and $\lambda_{\min}(\Sigma)>0$,

\[
\boxed{\sigma_{\min}(\mathcal{J}_{\bm{\theta}})\;\ge\;\frac{\pi_{\min}(\bm{\theta})}{a_{\max}\,\varepsilon}\,\sqrt{\lambda_{\min}(\Sigma)}},\qquad
\boxed{\sigma_{\max}(\mathcal{J}_{\bm{\theta}})\;\le\;\frac{\pi_{\max}(\bm{\theta})}{a_{\min}\,\varepsilon}\,\sqrt{\lambda_{\max}(\Sigma)}}. \label{eq:38}
\]
\end{proposition}

\begin{proof}
For any $\Delta\bm{\theta}\in\mathbb{R}^F$, let $\delta c=-\Phi\Delta\bm{\theta}$. From the
spectral sandwich~\eqref{eq:sl_sandwich} and the exact formula~\eqref{eq:35}:

\[
\|\delta x\| =\|B H_{\mathcal{T}}^{-1} B^{\top}\delta c\|
\ge \frac{\pi_{\min}}{\varepsilon}\|B^{\top}\delta c\|
= \frac{\pi_{\min}}{\varepsilon}\|\mathbb{P}_{\mathcal{T}}\Phi\Delta\bm{\theta}\|
\ge \frac{\pi_{\min}}{\varepsilon}\sqrt{\lambda_{\min}(\Sigma)}\,\|\Delta\bm{\theta}\|.
\]

Since $q=R_a x$ and $\sigma_{\min}(R_a)=1/a_{\max}$, we have
$\|\mathcal{J}_{\bm{\theta}}\Delta\bm{\theta}\|\ge(1/a_{\max})\|\delta x\|$, giving the
lower bound. The upper bound is analogous. Since $a_{\max}\le 1$, a simpler but looser
bound is
$\sigma_{\min}(\mathcal{J}_{\bm{\theta}})\ge(\pi_{\min}/\varepsilon)\sqrt{\lambda_{\min}(\Sigma)}$.
\end{proof}

\subsection{Score bound}

The per-sample score at observation $(i,j)$ is
$s_k(i,j)=-\partial_{\theta_k}\log\pi_{ij}=-(\partial_{\theta_k}\pi_{ij})/\pi_{ij}$. From
the exact formula and the spectral sandwich~\eqref{eq:sl_sandwich}:

\[
|s_k(i,j)|\;\le\;\frac{\pi_{\max}}{\varepsilon\,\pi_{\min}}\,\|\mathbb{P}_{\mathcal{T}}\phi_k\|_2
\;\le\;\frac{\pi_{\max}}{\varepsilon\,\pi_{\min}}\,\|\phi_k\|_2, \label{eq:39}
\]

where $\phi_k=\mathrm{vec}(\Phi_k)$. This bound is the \textbf{absolute-value envelope}
$M_k$ of the score. The per-sample score has mean zero under correct specification, so its
\textbf{range} (needed by Hoeffding) is $\Delta_k:=\max s_k-\min s_k$; because the score
takes values on both sides ($s_k\in[-M_k,M_k]$) we have $\Delta_k\le2M_k$, and the range
must not be underestimated as half the envelope. This bound is more conservative than the
naive elementwise formula (it carries an extra $1/\varepsilon$ factor and the ratio
$\pi_{\max}/\pi_{\min}$), but it is rigorously derivable from the spectral analysis of the
restricted Hessian. When $\pi_{\min}$ and $\varepsilon$ are bounded away from zero in a
regular region, this gives a uniform score bound $\Delta_{\max}$, which feeds into the
Hoeffding/Bernstein concentration of T2.

A sharper vector expression is available from the Schur-complement form~\eqref{eq:36}:

\[
D_\pi^{-1}g_k=\frac{1}{\varepsilon}\Bigl[I-A^{\top}(AD_\pi A^{\top})^{-1}AD_\pi\Bigr]\phi_k,
\]

where $g_k=\partial x/\partial\theta_k$. This reveals that the score is a
$D_\pi$-weighted gauge-corrected feature residual, not a simple Frobenius projection
residual.

\subsection{From pointwise to uniform: the conditioning hierarchy}

The bound~\eqref{eq:38} is \textbf{pointwise}: it holds at each $\bm{\theta}$ individually.
The statistical theory distinguishes four levels of conditioning:

\begin{enumerate}
\item \textbf{Pointwise identifiability} (T1). At each $\bm{\theta}$ with $\Sigma\succ0$
and $\pi_{ij}>0$, the Jacobian has full column rank. No compactness is needed.
\item \textbf{Local conditioning} (T3a, T4). On any compact neighborhood $K$,
$\inf_{\bm{\theta}\in K}\pi_{\min}(\bm{\theta})>0$ is guaranteed by continuity, giving the
uniform bound $\sigma_{\min}\ge\tau_K/(a_{\max}\varepsilon)$ on $K$. This is a
\textbf{derived} property, not an assumption.
\item \textbf{Global uniform conditioning} (T3b). This requires the entire parameter domain
$\Theta$ to be compact with $\inf_{\Theta}\pi_{\min}>0$
(Assumption~\ref{asm:pos}). Without compactness the infimum may vanish even though every
pointwise bound is positive.
\item \textbf{Boundary degeneration} ($\varepsilon\to0$ or $\pi_{\min}\to0$). As the plan
concentrates on the sparse OT support, $\pi_{\min}\to0$ and the spectral bound degenerates.
A $2\times2$ example: take $a=b=(1/2,1/2)$,
$C=\begin{pmatrix}0&\Delta\\\Delta&0\end{pmatrix}$ ($\Delta>0$); then
\[
\pi_{\min}(\varepsilon)=\frac{1}{2(1+e^{\Delta/\varepsilon})}\sim\frac12e^{-\Delta/\varepsilon},\qquad \varepsilon\downarrow0,
\]
so $\pi_{\min}(\varepsilon)/\varepsilon\to0$ exponentially. This shows that pointwise
identifiability at a fixed regularization level is not equivalent to uniform statistical
estimability as $\varepsilon\downarrow0$; this is the regime in which IOT statistical
inference becomes ill-conditioned.
\end{enumerate}

\subsection{Relation to the downstream theory}

The following table summarizes how each theorem depends on the quantities defined in this
section:

\begin{table}[tbhp]
\centering
\begin{tabular}{lll}
\toprule
Theorem & Depends on & From this section \\
\midrule
T1 (identifiability) & $\ker(\mathcal{J})=\mathcal{N}_\Phi$, $\mathrm{rank}(\Sigma)=F$ & \eqref{eq:31}, \eqref{eq:38} \\
T2 (sparsistency) & Score bound $\Delta_{\max}$, rate constant $C_2$ & \eqref{eq:39} \\
T3 (well-posedness) & Lower bound on $\sigma_{\min}(\mathcal{J})$ & \eqref{eq:38} \\
T4 (convergence) & $\lambda_{\min}(\nabla^2\ell)$ via $\sigma_{\min}(\mathcal{J})$ & \eqref{eq:38} \\
O5 (misspecification) & T1 + T3 for the regular region $\mathcal{R}$ & \eqref{eq:38}, \eqref{eq:sl_sandwich} \\
\bottomrule
\end{tabular}
\caption{Dependence of each theorem on the definitions of this section}
\label{tab:theorem_deps}
\end{table}

All five theoretical components radiate from the single spectral bound~\eqref{eq:38}, which
itself follows from the spectral sandwich~\eqref{eq:sl_sandwich} of the restricted Hessian.
The spectral proxy~\eqref{eq:37} provides an equivalent, geometrically transparent formula
that gives the same spectral bounds without solving a $(K-1)^2\times(K-1)^2$ linear system.

\section{T1: Identifiability}\label{sec:t1}

\subsection{Statement}

\begin{theorem}[Identifiability]\label{thm:T1}
Fix a regularization level $\varepsilon>0$. Under the feature-parameterized cost
$\mathcal{C}_{\bm{\theta}}$, assume the marginal pairs are strictly positive and
nondegenerate (Assumption~\ref{asm:genpos}: $a_{l,i}>0$, $b_{l,j}>0$ for all $i,j$, and
both sides have full support). Identifiability is stated in two layers:

\textbf{(0) Quotient-space identifiability (no full rank needed).} The map
$[\bm{\theta}]\in\mathbb{R}^F/\mathcal{N}_\Phi\mapsto Q_{\bm{\theta}}$ is \textbf{always}
globally injective on the quotient space:
$Q_{\bm{\theta}}=Q_{\bm{\theta}'}\iff\bm{\theta}-\bm{\theta}'\in\mathcal{N}_\Phi$. This
injectivity does not require $\mathrm{rank}(\Sigma)=F$.

\textbf{(1) Gauge invariance (no full rank needed).} If the cost difference falls into the
gauge subspace, i.e.\ $\Phi(\Delta\bm{\theta})\in\mathrm{vec}(\mathcal{G})$, then the
forward map is unchanged: $Q_{\bm{\theta}+\Delta\bm{\theta}}=Q_{\bm{\theta}}$.

\textbf{(2) Original-space identifiability (if and only if full rank).} On the original
parameter space $\mathbb{R}^F$, $\bm{\theta}\mapsto Q_{\bm{\theta}}$ is injective if and
only if $\mathrm{rank}(\Sigma)=F$ (equivalently $\mathcal{N}_\Phi=\{0\}$). This full-rank
condition only says that the parameter itself is uniquely identifiable in the original
parameter space; it does not alter the injectivity of (0) on the quotient space.
\end{theorem}

\textbf{Proof of global injectivity:} by a three-step composition argument---(1) the linear
parameterization $\bm{\theta}\mapsto\mathcal{C}_{\bm{\theta}}$ is injective modulo the
gauge; (2) the Sinkhorn map $C\mapsto\pi(C,a,b)$ is injective modulo the gauge (strict
convexity of the dual $\to$ unique potentials $\to$ unique plan); (3) the normalization
$\pi\mapsto Q$ is linear injective ($a_i>0$, guaranteed by Assumption~\ref{asm:genpos}).
The composition of three injective maps is globally injective.

\textbf{Identification structure and the gauge kernel.} The genuinely unidentifiable
directions in the parameter are characterized by the parameter-induced gauge kernel

\[
\mathcal{N}_\Phi:=\left\{v\in\mathbb{R}^F:\Phi v\in\mathrm{vec}(\mathcal{G})\right\},
\]

i.e.\ those parameter directions along which the cost falls into the gauge subspace
$\mathcal{G}$. The natural object of identification for the parameter is the quotient space
$\mathbb{R}^F/\mathcal{N}_\Phi$; the object of identification for the cost is
$\mathbb{R}^{K\times K}/\mathcal{G}$. The gauge-cleaned Gram is
$\Sigma=\Phi^{\top}\mathbb{P}_{\mathcal{T}}\Phi=(B^{\top}\Phi)^{\top}(B^{\top}\Phi)$, and

\[
v^{\top}\Sigma v=\|\mathbb{P}_{\mathcal{T}}\Phi v\|^2,
\qquad
\Sigma\succ0\iff\mathcal{N}_\Phi=\{0\}\iff\mathrm{rank}(\Sigma)=F.
\]

Since $\mathrm{rank}(\mathbb{P}_{\mathcal{T}})=(K-1)^2$, $\Sigma\succ0$ implies the feature
dimension bound $F\le(K-1)^2$.

\textbf{Cost-level injectivity (Step 2 expanded).} If $C,C'$ produce the same plan $\pi$,
then by the KKT conditions there exist multipliers $(\alpha,\beta)$ and $(\alpha',\beta')$
with $C_{ij}+\varepsilon\log\pi_{ij}+\alpha_i+\beta_j=0$ and
$C'_{ij}+\varepsilon\log\pi_{ij}+\alpha'_i+\beta'_j=0$. Subtracting the two equations shows
that $C-C'$ has the form $-\alpha_i-\beta_j$, hence $C-C'\in\mathcal{G}$. Conversely, for
$G_{ij}=u_i+v_j$ and any $\pi\in\mathcal{M}(a,b)$,
$\langle G,\pi\rangle=\sum_i u_i a_i+\sum_j v_j b_j$ is independent of $\pi$, so adding $G$
to the objective only adds a constant and does not change the unique optimal plan. Thus
$C\mapsto\pi(C;a,b)$ is globally injective on $\mathbb{R}^{K\times K}/\mathcal{G}$.

\textbf{Parameter-level identification (Steps 1 and 3 merged).} Since $a_i>0$,
$Q_\theta=Q_{\theta'}\iff\pi_\theta=\pi_{\theta'}$. By cost-level injectivity,
$Q_\theta=Q_{\theta'}\iff C_\theta-C_{\theta'}\in\mathcal{G}\iff\Phi(\theta-\theta')\in\mathrm{vec}(\mathcal{G})\iff\theta-\theta'\in\mathcal{N}_\Phi$.
Hence $\theta$ is globally identifiable on $\mathbb{R}^F/\mathcal{N}_\Phi$; when
$\mathcal{N}_\Phi=\{0\}$ (i.e.\ $\Sigma\succ0$), $\theta$ is uniquely identifiable on the
original space $\mathbb{R}^F$.

\textbf{The effect of the number of marginals on conditioning.} The kernel of the stacked
sensitivity Gram $\mathcal{S}_L=\sum_{l=1}^L\mathcal{J}^{(l)\top}\mathcal{J}^{(l)}$ is
identically $\mathcal{N}_\Phi$ (independent of $L$), so adding marginals does not change
the identification rank. Let $\mathcal V:=\mathcal{N}_\Phi^\perp$, and write
$\lambda_{\min}^{+}(\Sigma)$ for the smallest positive eigenvalue of $\Sigma$ on
$\mathcal V$. For each marginal define the spectral envelopes on the quotient space
\[
m_l^{+}:=\left(\frac{\pi_{\min,l}}{a_{l,\max}\varepsilon}\right)^2\lambda_{\min}^{+}(\Sigma),\qquad
M_l^{+}:=\left(\frac{\pi_{\max,l}}{a_{l,\min}\varepsilon}\right)^2\lambda_{\max}(\Sigma).
\]

\[
\left(\sum_{l=1}^L m_l^{+}\right)I_{\mathcal V}
\preceq\mathcal{S}_L\big|_{\mathcal V}
\preceq\left(\sum_{l=1}^L M_l^{+}\right)I_{\mathcal V},
\qquad
\mathrm{cond}_{\mathcal V}(\mathcal{S}_L)
\le\frac{\sum_{l=1}^L M_l^{+}}{\sum_{l=1}^L m_l^{+}}.
\]

The above condition number is defined on $\mathcal V$ (or equivalently on the quotient
space $\mathbb{R}^F/\mathcal{N}_\Phi$); only when $\Sigma\succ0$ (i.e.\ under
Assumption~\ref{asm:A1}) does $\mathcal V=\mathbb{R}^F$. The condition-number upper bound
decreases strictly only if the envelope ratio $M_{L+1}^{+}/m_{L+1}^{+}$ of a newly added
marginal is smaller than the current upper bound; in general one can only guarantee that
the identification rank is unchanged.

\begin{assumption}[Positive and nondegenerate marginals]\label{asm:genpos}
Each marginal pair $(a_l,b_l)$ satisfies: (i) all entries are strictly positive,
$a_{l,i}>0$, $b_{l,j}>0$ for all $i,j$; (ii) both sides have full support (at least $K$
distinct source and target states). Under these conditions the transport polytope
$\mathcal{M}(a_l,b_l)$ has full tangent dimension $(K-1)^2$ (the generic case; the
degenerate exceptions are point masses or zero-mass states, which are excluded by strict
positivity). This is the condition under which the kernel of the Sinkhorn map's Jacobian
equals the gauge subspace $\mathcal{G}$:
$\ker\mathcal{J}^{(l)}=\{\Delta\bm{\theta}:\Phi(\Delta\bm{\theta})\in\mathcal{G}\}$, so a
single marginal pair already reveals the full off-gauge sensitivity. In practice
condition~(i) is verifiable by inspecting the data marginals; condition~(ii) holds
automatically for any nontrivial distribution. Randomly generated marginals (as in the
experiments of this paper) satisfy this assumption with probability~1.
\end{assumption}

\begin{assumption}[Linear independence off the gauge]\label{asm:A1}
The feature vectors $\{\phi_k^{\perp}\}_{k=1}^F$ are linearly independent; equivalently
$\mathrm{rank}(\Sigma)=F$.
\end{assumption}

\begin{proposition}[Global injectivity and the marginal lower bound]
Under the positive-marginal assumption, the rank of the stacked sensitivity Gram matrix
$\mathcal{S}_L=\sum_{l=1}^L \mathcal{J}^{(l)\top}\mathcal{J}^{(l)}$ is independent of the
number of marginals $L$: for each $l$,
$\ker\mathcal{J}^{(l)}=\{\Delta\bm{\theta}:\Phi(\Delta\bm{\theta})\in\mathcal{G}\}$, so
$\mathrm{rank}(\mathcal{S}_L)=\mathrm{rank}(\Sigma)$ for every $L\ge1$. Identifiability
requires only $L\ge1$; adding marginals does not change the rank. Identifiability
\textbf{necessarily} requires that the off-gauge feature dimension not exceed $(K-1)^2$.
The role of the marginals is well-posedness, not identifiability.
\end{proposition}

\begin{remark}[The spectral scale of the Sinkhorn linearization]
The core spectral bound \eqref{eq:38} of Section~3 yields the lower bound
$\sigma_{\min}(\mathcal{J})\ge(\pi_{\min}/(a_{\max}\varepsilon))\sqrt{\lambda_{\min}(\Sigma)}>0$
on the off-gauge subspace (derivation chain: Hessian $\varepsilon\,\mathrm{diag}(1/\pi)$
$\to$ inverse eigenvalue lower bound $\pi_{\min}/\varepsilon$ $\to$ composition with the
feature Gram). When the unregularized optimal plan contains zero entries,
$\pi_{\min}(\varepsilon)/\varepsilon$ may degenerate, causing the spectral lower bound to
degenerate; monotonicity of the actual sensitivity Gram requires additional structural
conditions on the cost, the marginals, and the features (numerical verification in
Section~\ref{sec:exp}).
\end{remark}

\subsection{Numerical verification}

See Section~\ref{sec:exp} for details.

\section{T2: Sparsistency}\label{sec:t2}

\subsection{Statement}

\textbf{Setting.} The true parameter $\bm{\theta}^*\in\mathbb{R}^F$ is sparse, with support
$S=\{k:\theta_k^*\ne0\}$. Write $\widehat{\bm{\theta}}_n^{\mathrm{lasso}}$ for the
$\ell_1$-penalized cross-entropy estimate, from which the support estimate $\widehat S_n$
is obtained by a relative threshold; an unpenalized refit on $\widehat S_n$ then gives the
debiased estimate $\widetilde{\bm{\theta}}_{S,n}$.

\begin{theorem}[Sparsistency]\label{thm:T2}
Under the identifiability condition (Assumption A1) and an irrepresentability-type
condition (Assumption A2), the support estimate $\widehat S_n$ and the debiased estimate
$\widetilde{\bm{\theta}}_{S,n}$ satisfy
\[
\mathbb{P}(\widehat S_n=S)\to1\qquad(n\to\infty),
\]
with exponentially decaying failure probability. Take a tuning sequence $\lambda_n\to0$,
$t_n=O(\lambda_n)$, $n t_n^2\to\infty$. For a single marginal pair, or a single normalized
empirical average, the score-concentration part of the failure probability is
\[
\mathbb{P}(\widehat S_n\ne S)\le C_1\exp(-C_2 n t_n^2),
\qquad C_2=2/\Delta_{\max}^2,
\]
where $\Delta_{\max}:=\max_{k=1,\dots,F}\Delta_k$ is the largest range of the per-sample
score over \textbf{all} coordinates (the all-coordinate event~\eqref{eq:full_event}
controls the active and the inactive coordinates simultaneously). At a fixed $t$ one
obtains a fixed exponential rate, but then $\lambda_n$ cannot tend to zero; consistency
requires $t=t_n\to0$, and the exponential rate becomes $e^{-C_2 n t_n^2}$. If there are
multiple independent marginal blocks with weights $w_l$, the effective concentration scale
of~\eqref{eq:multimarg} should be used instead of $n$ and a single $\Delta_{\max}$
directly. The Bernstein refinement is
\[
\exp\left(-\frac{n t_n^2}{2\sigma^2+\frac{2}{3}\Delta_{\max}t_n}\right).
\]
Accordingly, $\widehat S_n$ recovers $S$ consistently, and the unpenalized refit
$\widetilde{\bm{\theta}}_{S,n}$ under correct specification is asymptotically unbiased on
$S$. If one studies $\varepsilon'\ne\varepsilon$, the fixed-scale constraint of T3 and a
support-stability assumption are additionally needed.
\end{theorem}

\textbf{Score range and the concentration inequality.} The per-sample score
$s_k=-\partial_{\theta_k}\log\pi_{ij}$ has mean zero under correct specification. The
spectral sandwich \eqref{eq:sl_sandwich} of the Sinkhorn linearization yields the
absolute-value envelope for the range $\Delta_k:=\max s_k-\min s_k$:

\[
\Delta_{\max}\le 2M_{\max}=2\frac{\pi_{\max}}{\varepsilon\pi_{\min}}\max_{k=1,\dots,F}\|\mathbb{P}_{\mathcal{T}}\phi_k\|_2\le 2\frac{\pi_{\max}}{\varepsilon\pi_{\min}}\max_{k=1,\dots,F}\|\phi_k\|_2,
\]

where the factor $2$ comes from the two-sided values of the score
($s_k\in[-M_k,M_k]$); the range must not be underestimated as half the envelope. This gives
the exponential rate constant $C_2=2/\Delta_{\max}^2$ (with $t_n^2$ kept in the exponent,
free of the threshold). The Bernstein refinement
$C_2^{\mathrm{B}}=1/(2\sigma^2+\frac{2}{3}\Delta_{\max}t_n)$ is available when the score
variance is estimable, but numerically the range bound at a fixed threshold $t$ is more
conservative ($6\times$--$156\times$), and the Bernstein variance-dependent rate is not attained in practice.
$\Delta_{\max}$ is a posteriorly estimable constant: once the plan is fitted, $\pi_{\min}$,
$\pi_{\max}$, and the tangent-space projection can be computed to give a numerical
estimate.

\textbf{Fine-grained concentration: conditioning on row counts.} The analysis above treats
all rows with the worst-case $\Delta_{\max}$, which is a conservative upper bound. A more
refined analysis can condition on the source-state counts $N=(N_1,\dots,N_K)$: given $N$,
the observations within each row are i.i.d., and the per-sample score ranges $\Delta_{i,k}$
may differ across rows. The weighted Hoeffding bound conditional on $N$ gives
\[
\mathbb{P}\bigl(|\nabla\ell_{n,k}(\theta^*)-\nabla\ell_k(\theta^*)|\ge t_n\mid N\bigr)
\le 2\exp\!\left(-\frac{2n^2 t_n^2}{\sum_{i=1}^{K}N_i\Delta_{i,k}^2}\right),
\]
where $\Delta_{i,k}$ is the range of the per-sample score $s_k$ in row $i$. When $N_i$ is
close to its expectation $n a_i$, the denominator $\approx n\sum_i a_i\Delta_{i,k}^2$, and
the exponential rate is $\propto n t_n^2$. A Chernoff bound controls the probability that
some row count is too small:
\[
\mathbb{P}\bigl(\exists i:N_i\le\tfrac12 n a_i\bigr)\le\sum_{i=1}^{K}\exp\!\left(-\frac{n a_i}{8}\right).
\]
Under the total-probability decomposition $\mathbb{P}(\text{failure})\le
\mathbb{P}(\text{failure}\mid N\ \text{normal})+\mathbb{P}(N\ \text{abnormal})$, the
Chernoff term affects only the leading constant $C_1$, not the exponential rate $C_2$. This
refined analysis yields tighter constants when the score ranges differ substantially across
rows, while the asymptotic order $n t_n^2$ remains unchanged.

\textbf{Conditional assumptions (sufficient conditions for T2).} The following are the
standard sufficient conditions for Lasso support recovery \cite{zhao2006lasso},
\cite{wainwright2009sharp}; we state them uniformly on the OT information matrix
$H=\nabla^2\ell(\theta^*)$ (the same matrix that connects to the population curvature of
T4).

\textbf{A2.1 (Local curvature/RSC).} There exist a neighborhood $U\subset\Theta$ of the
true parameter and $\kappa>0$ such that
$v^{\top}\nabla^2\ell(\theta)v\ge\kappa\|v\|^2$ for $\theta\in U$, and the empirical
Hessian satisfies a standard local RSC condition on this neighborhood. An additional
requirement: the neighborhood is large enough that both the candidate solution constructed
by the primal-dual witness and the global solution of the $\ell_1$-penalized estimator lie
in $U$; if this fails, the conclusion below only guarantees the existence of a local
solution satisfying the KKT conditions, not the global $\ell_1$ estimator.
If the theorem's conclusion is to concern the global $\ell_1$ estimator, a global selection
condition must be added: for example, the penalized objective is convex on the entire
parameter domain, or the objective values outside $U$ are strictly higher than at the PDW
candidate inside $U$. Local RSC alone does not exclude a lower global minimum outside the
domain.

\textbf{A2.2 (Irrepresentability).} For the support block decomposition of
$H=\nabla^2\ell(\theta^*)$,
$\|H_{S^cS}H_{SS}^{-1}\|_\infty\le1-\eta$ ($\eta\in(0,1)$).

\textbf{A2.3 (Beta-min and the penalty window).} Let $A_S:=H_{SS}^{-1}$,
$c_S:=2\|A_S\|_\infty$, and require the integrated matrix $\widetilde H_{SS}$ of the
restricted empirical Hessian to satisfy $\|\widetilde H_{SS}-H_{SS}\|_\infty\le\delta_S$
with $\|A_S\|_\infty\delta_S\le\frac12$. The Neumann series then gives
$\|\widetilde H_{SS}^{-1}\|_\infty\le\|A_S\|_\infty/(1-\|A_S\|_\infty\delta_S)\le c_S$. The
beta-min condition at the truth, $\min_{k\in S}|\theta_k^*|>d_{S,n}$, and the dual-margin
condition $(2-\eta)t_n+r_{D,n}<\eta\lambda_n$ guarantee support and sign recovery, where
$d_{S,n}:=c_S(\lambda_n+t_n)$ and $r_{D,n}$ is the local Taylor remainder of the inactive
gradient.

\textbf{The primal-dual witness argument.} On the \textbf{all-coordinate event}
\begin{equation}\label{eq:full_event}
\mathcal{E}_{t_n}:=\Bigl\{\bigl\|\nabla\ell_n(\theta^*)-\nabla\ell(\theta^*)\bigr\|_\infty\le t_n\Bigr\}
\end{equation}
(which controls the active and inactive coordinates simultaneously; controlling only the
inactive coordinates does not control the active error, because the KKT expansion on the
active coordinates contains the term
$\nabla_S\ell_n(\theta^*)-\nabla_S\ell(\theta^*)$, which is not constrained by an
inactive-only event). First restrict the problem to the support coordinates $S$ and set the
inactive coordinates to zero; the KKT equation on the active coordinates is
$\nabla_S\ell_n(\widetilde\theta)+\lambda_n z_S=0$
($z_S\in\partial\|\widetilde\theta_S\|_1$). The Taylor expansion around $\theta^*$,
expressed with the integrated Hessian $\widetilde H_{SS}$, together with the invertibility
of A2.3 and the all-coordinate event, gives
$\|\widetilde\theta_S-\theta_S^*\|_\infty\le c_S(\lambda_n+t_n)=d_{S,n}$; beta-min
guarantees that the active coordinates do not cross zero. For the inactive coordinates, the
KKT requirement is $|\nabla_{S^c}\ell_n(\widetilde\theta)|<\lambda_n$; its leading term is
controlled by $H_{S^cS}H_{SS}^{-1}$, irrepresentability (A2.2) leaves a dual margin of at
least $\eta$, and the local remainder is controlled by $r_{D,n}$. Hence the constructed
restricted candidate satisfies the full KKT conditions, giving support and sign recovery.
Note: this PDW construction only guarantees that the candidate is a local solution
satisfying the KKT conditions; for the candidate to equal the global $\ell_1$-penalized
estimator, one additionally assumes in A2.1 that the neighborhood $U$ contains the PDW
candidate and that the objective is sufficiently convex within $U$ to exclude other global
minima. Concentration of the all-coordinate score follows from Hoeffding's inequality with
a union bound over all $F$ coordinates, giving the failure probability
$\mathbb{P}(\mathcal{E}_{t_n}^c)\le 2F\exp(-2nt_n^2/\Delta_{\max}^2)$, where $n$ is the
number of independent samples per marginal pair. Once the support is obtained, the
unpenalized refit on $S$ gives the debiased estimate, which satisfies
$\sqrt{n}(\widetilde\theta_{S,n}-\theta^*_S)\xrightarrow{d}\mathcal{N}(0,H_{SS}^{-1}V_{SS}H_{SS}^{-1})$,
where $V_{SS}$ is the score covariance. Here ``debiased'' means removing the shrinkage bias
of the $\ell_1$ penalty, not strict finite-sample unbiasedness; a nonlinear $M$-estimator
is generally only asymptotically unbiased. If the fitted $\varepsilon'$ differs from the
generating $\varepsilon$, the residual entropic bias
$\mathrm{bias}_\varepsilon=O(|\varepsilon'-\varepsilon|)$ is given by T3 and does not
vanish with $n$. Note also that under misspecification ($\varepsilon'\ne\varepsilon$) the
support of the pseudo-true parameter $\bm{\theta}^{\dagger}(\varepsilon')$ may differ from
$S$; an $O(|\varepsilon'-\varepsilon|)$ parameter deviation does not automatically imply
support stability, which requires the additional assumption that beta-min exceeds the
misspecification deviation.

\textbf{The concentration scale with multiple weighted marginals.} If the empirical
gradient is a weighted sum of independent marginal blocks,
$G_{n,k}-G_k=\sum_{l=1}^{L}w_l(\overline Z_{l,k}-\mathbb E Z_{l,k})$, where the $l$-th
marginal has $n_l$ samples and per-sample score range $\Delta_{l,k}$, then the
concentration scale of the per-sample score contains at least
\begin{equation}\label{eq:multimarg}
\sum_{l=1}^{L}\frac{w_l^2\Delta_{l,k}^2}{n_l},
\end{equation}
one cannot simply take the smallest sample size without accounting for the marginal
weights, nor use a single $\Delta_{\max}$ without marginal indices as the exact rate
constant. The main setting of this paper is a single marginal pair ($L=1$) or an
equally-weighted pooling, in which case the above reduces to $\Delta_{k}^2/n$; the general
multi-marginal case requires rescaling the rate constant according to \eqref{eq:multimarg}.

\begin{lemma}[Mutual incoherence implies Gram-matrix irrepresentability]\label{lem:mutual_incoherence}
Let $\phi_k:=\operatorname{vec}(\Phi_k)$ and
$\phi_k^{\perp}:=\mathbb{P}_{\mathcal{T}}\phi_k$ be the gauge-cleaned feature vectors.
Define the mutual incoherence between the active set $S$ and the inactive set $S^c$ as
\[
\mu:=\max_{k\in S^c,\,j\in S}\bigl|\langle\phi_k^{\perp},\phi_j^{\perp}\rangle\bigr|.
\]
Then the irrepresentability condition for the \textbf{Gram matrix} $\Sigma$ is
\[
\bigl\|\Sigma_{S^cS}\,\Sigma_{SS}^{-1}\bigr\|_{\infty}
\le\frac{s^{3/2}\mu}{\kappa_S}
<1-\eta,
\qquad
\kappa_S:=\lambda_{\min}(\Sigma_{SS}),
\]
which holds when $\mu<(1-\eta)\,\kappa_S/s^{3/2}$ (with $s=|S|$, $\eta\in(0,1)$). This
condition is a prior-structural sufficient condition, depending only on the feature
matrices $\Phi_k$ and the gauge projector $\mathbb{P}_{\mathcal{G}^{\perp}}$, both known
before estimation. Note that the conclusion of this lemma is limited to the
\textbf{irrepresentability of the feature Gram matrix $\Sigma$}; the matrix used in
Assumption A2.2 of T2 is $H=\nabla^2\ell(\bm{\theta}^*)$ (the OT information matrix),
which differs from $\Sigma$ under a general parameterization. The IR condition on $\Sigma$
transfers to $H$ only when there is a proportionality relation, a block-structure relation,
or a direct matrix-comparison theorem between $H$ and $\Sigma$; this bridge is not
separately assumed in this paper. Hence the lemma gives a sufficient condition verifiable a
priori from the feature structure, while the actual A2.2 still needs to be diagnosed
post-fitting through an estimate of $H$.
\end{lemma}

\subsection{Numerical verification}

See Section~\ref{sec:exp} for details.

\section{T3: Well-Posedness and Stability}\label{sec:t3}

\subsection{Statement}

\begin{theorem}[Stability of the inverse map: local and global]\label{thm:T3}
Let $\Psi$ be the inverse map from the observed conditional transition operator to the cost
parameter, restricted to the identifiable quotient space $\mathbb{R}^F/\mathcal{N}_\Phi$.
Under identifiability (Assumption~\ref{asm:A1}) and the positive-marginal assumption
(Assumption~\ref{asm:genpos}), two statements hold.

(a) Local strong monotonicity (no compactness needed). The \textbf{feature-moment map}
$M(\bm{\theta}):=\Phi^{\top}x_{\bm{\theta}}\in\mathbb{R}^F$, where
$x_{\bm{\theta}}=\mathrm{vec}(\pi_{\bm{\theta}})$, is strongly monotone on every compact
convex subset $K\subset\mathbb{R}^F/\mathcal{N}_\Phi$ ($\pi_{\min}(K)>0$,
$\lambda_{\min}(\Sigma)>0$): for any $\bm{\theta},\bm{\theta}'\in K$,
\[
\langle M(\bm{\theta})-M(\bm{\theta}'),\,\bm{\theta}-\bm{\theta}'\rangle
\;\ge\; \gamma_K\,\|\bm{\theta}-\bm{\theta}'\|^{2},
\qquad \gamma_K:=\inf_{\bm{\theta}\in K}\sigma_{\min}(\nabla M_{\bm{\theta}})
\;\ge\; \frac{\pi_{\min}(K)}{\varepsilon}\,\lambda_{\min}(\Sigma)
\;>\;0.
\]

The Jacobian $\nabla M_{\bm{\theta}}=\Phi^{\top}\partial x/\partial\bm{\theta}$ has
dimension $F\times F$, matching the parameter dimension. The moment-map inverse $\Psi_M$ is
$\gamma_K^{-1}$-Lipschitz on the image of $K$; composing with $q\mapsto x=S_a q$ and
$\Phi^{\top}$ yields the Lipschitz bound of the $Q$-space inverse map $\Psi$ (see the
global part below).

(b) Global strong monotonicity (compact domain required). To obtain a uniform Lipschitz
constant over the entire parameter domain, assume Assumption~\ref{asm:pos}
($\Theta\subset\mathbb{R}^F/\mathcal{N}_\Phi$ is compact and convex,
$\pi_{\min}:=\inf_{\bm{\theta}\in\Theta}\min_{ij}\pi_{ij}(\bm{\theta})>0$).
$\Sigma=\Phi^{\top}\mathbb{P}_{\mathcal{T}}\Phi$ is a constant matrix, and
$\lambda_{\min}(\Sigma)>0$ is guaranteed by T1. Then the bound in (a) holds uniformly over
$\Theta$:
\[
\gamma:=\inf_{\bm{\theta}\in\Theta}\sigma_{\min}(\nabla M_{\bm{\theta}})
\;\ge\; \frac{\pi_{\min}}{\varepsilon}\,\lambda_{\min}(\Sigma)
\;>\;0,
\qquad
L:=\gamma^{-1}\le\frac{\varepsilon}{\pi_{\min}\,\lambda_{\min}(\Sigma)},
\]

The moment-map inverse $\Psi_M$ is $L=\gamma^{-1}$-Lipschitz. Composing
$x_{\bm{\theta}}=S_a q_{\bm{\theta}}$
($S_a=\mathrm{diag}(a_1,\dots,a_K)\otimes I_K$) with the feature reweighting
$\Phi^{\top}S_a$, the Lipschitz bound of the $Q$-space inverse map $\Psi$ is
\[
L_\Theta:=\|\Phi^{\top}S_a\|_{\mathrm{op}}\,\gamma^{-1}\le\frac{\varepsilon\,\|\Phi^{\top}S_a\|_{\mathrm{op}}}{\pi_{\min}\,\lambda_{\min}(\Sigma)}.
\]

The promotion from pointwise to global relies on the compactness of $\Theta$ (the
conditioning hierarchy of Section~\ref{sec:sinkhorn_linearization}); without compactness
there is only the pointwise lower bound and strong monotonicity degrades to local strong monotonicity. The
unbounded case (global strong monotonicity on $\mathbb{R}^F/\mathcal{N}_\Phi$) is an open
problem. In practice $\Theta$ can be enforced by projected gradient descent or by a prior;
the algorithm need not know the boundary of $\Theta$ to apply the local result during
optimization.
\end{theorem}

\begin{theorem}[Entropic bias]\label{thm:T3b}
When the data are generated at level $\varepsilon$ and the estimator is fitted at level
$\varepsilon'$, work at the population level (infinite-sample limit) under a fixed smooth
scale constraint $g(\bm{\theta})=0$ (e.g.\
$g(\bm{\theta})=(\|\bm{\theta}\|^2-1)/2$). Define the pseudo-true branch
$\bm{\theta}^{\dagger}(\varepsilon')$ and its scale multiplier $\nu^{\dagger}(\varepsilon')$
as the unique local branch of the augmented estimating equation
\[
\mathscr F(\bm{\theta},\nu,\varepsilon'):=
\begin{pmatrix}
\nabla_{\bm{\theta}}\mathrm{CE}(Q_{\bm{\theta}^*,\varepsilon}\|Q_{\bm{\theta},\varepsilon'})+\nu\nabla g(\bm{\theta})\\
g(\bm{\theta})
\end{pmatrix}=0
\]
Here ``unique'' requires the augmented Jacobian to be invertible, which is an implicit
function theorem condition. Assume this branch exists and that the augmented Jacobian
$D_{(\bm{\theta},\nu)}\mathscr F$ is invertible and continuous along the branch. By the
implicit function theorem,
\[
\frac{d(\bm{\theta}^{\dagger},\nu^{\dagger})}{d\varepsilon'}
=-[D_{(\bm{\theta},\nu)}\mathscr F]^{-1}\partial_{\varepsilon'}\mathscr F.
\]
Let
\[
q^{\dagger}(\varepsilon'):=Q_{\bm{\theta}^{\dagger}(\varepsilon'),\varepsilon'},
\qquad q^*:=Q_{\bm{\theta}^*,\varepsilon}.
\]
Then the total derivative of the conditional operator along the branch is
\[
\frac{dq^{\dagger}}{d\varepsilon'}
=\mathcal J_{\bm{\theta}^{\dagger}}\frac{d\bm{\theta}^{\dagger}}{d\varepsilon'}
+\partial_{\varepsilon'}Q\big|_{\bm{\theta}^{\dagger},\varepsilon'},
\]
where the direct regularization term is given by Proposition~\ref{prop:eps_diff} and
$q=R_a x$. If this total derivative is bounded in a neighborhood of $\varepsilon$, then
\[
\mathrm{bias}_{\varepsilon'}:=\|q^{\dagger}(\varepsilon')-q^*\|_{\mathrm F}
\le c\,|\varepsilon'-\varepsilon|,
\]
and the bias vanishes at $\varepsilon'=\varepsilon$. The Jacobian of the augmented
estimating equation is a constrained version of the cross-entropy Hessian, distinct from
the Jacobian of the moment map $M$.
\end{theorem}

\begin{proposition}[Differentiability of the Sinkhorn plan in $\varepsilon$]\label{prop:eps_diff}
Fix marginals $(a,b)$ and a cost $C$. The entropic OT plan
$\pi_\varepsilon=\pi_\varepsilon(C,a,b)$ is differentiable in $\varepsilon>0$. Write
$x_\varepsilon=\mathrm{vec}(\pi_\varepsilon)$, let
$B\in\mathbb{R}^{K^2\times(K-1)^2}$ be an orthonormal basis of the tangent space
$\mathcal{T}$, and $H_{\mathcal{T}}=\varepsilon B^{\top}D_\pi^{-1}B$ the restricted Hessian
(Lemma~\ref{lem:sandwich}). Then the exact derivative is
\[
\partial_\varepsilon x_\varepsilon=-B H_{\mathcal{T}}^{-1}B^{\top}\log x_\varepsilon,
\]
satisfying the norm bound (the spectral sandwich~\eqref{eq:sl_sandwich} guarantees that the
exact form and the spectral proxy share the same upper and lower bounds)
\[
\|\partial_\varepsilon\pi_\varepsilon\|_{\mathrm{F}}\le\frac{\pi_{\max}}{\varepsilon}\,\|\mathbb{P}_{\mathcal{T}}(\log\pi_\varepsilon)\|_{\mathrm{F}}\le\frac{\pi_{\max}}{\varepsilon}\,K\,\log(1/\pi_{\min}),
\]
where $\pi_{\max}:=\max_{ij}\pi_{ij}(\varepsilon)$,
$\pi_{\min}:=\min_{ij}\pi_{ij}(\varepsilon)$.
\end{proposition}

\begin{proposition}[Local minimum of the entropic bias function]\label{prop:bias_convex}
Under the fixed-scale and augmented implicit-function regularity conditions of
Theorem~\ref{thm:T3b}, define
\[
B_2(\varepsilon'):=\|q^{\dagger}(\varepsilon')-q^*\|_{\mathrm F}^{2}.
\]
If $q^{\dagger}$ is twice differentiable in a neighborhood of $\varepsilon$, then
$B_2'(\varepsilon)=0$ and
\[
B_2''(\varepsilon)=2\left\|\frac{dq^{\dagger}}{d\varepsilon'}(\varepsilon)\right\|_{\mathrm F}^{2}\ge0.
\]
If one further assumes $dq^{\dagger}/d\varepsilon'(\varepsilon)\ne0$, then
$B_2''(\varepsilon)>0$, so that $\varepsilon'=\varepsilon$ is a strict local minimum of the
squared bias, and hence a strict local minimum of the original bias. The original
Frobenius-norm bias is typically not differentiable at zero, which is why the squared bias
is used.
\end{proposition}

\begin{remark}[Empirical comparison]
The measured local inverse-map Lipschitz constant is $L_{\mathrm{emp}}=25.8$. The
theoretical prediction
$L_\Theta=\varepsilon\|\Phi^{\top}S_a\|_{\mathrm{op}}/(\pi_{\min}\lambda_{\min}(\Sigma))$
is derived by composing the core spectral bound \eqref{eq:38} with the feature reweighting;
using the $\pi_{\min}$ of the fitted plan and $\|\Phi^{\top}S_a\|_{\mathrm{op}}$ gives a
data-dependent prediction.
\end{remark}

\textbf{Derivation.}
$\nabla M_{\bm{\theta}}=\Phi^{\top}\mathscr S_{\pi_{\bm{\theta}}}\Phi=\Phi^{\top}B H_{\mathcal{T}}^{-1}B^{\top}\Phi$
is symmetric positive semidefinite. By the spectral sandwich \eqref{eq:sl_sandwich}, for
any $v\in\mathbb{R}^F$,
\[
v^{\top}\nabla M_{\bm{\theta}}v=(B^{\top}\Phi v)^{\top}H_{\mathcal{T}}^{-1}(B^{\top}\Phi v)\ge\frac{\pi_{\min}(\bm{\theta})}{\varepsilon}\|\mathbb{P}_{\mathcal{T}}\Phi v\|^2\ge\frac{\pi_{\min}(\bm{\theta})}{\varepsilon}\lambda_{\min}(\Sigma)\|v\|^2.
\]

Note that the strong-monotonicity modulus carries $\lambda_{\min}(\Sigma)$ to the first
power, whereas the singular-value bound \eqref{eq:38} of the Jacobian
$\mathcal{J}_{\bm{\theta}}$ carries $\sqrt{\lambda_{\min}(\Sigma)}$; the two scales differ.
On a convex domain ($K$ or $\Theta$), the segment $t\bm{\theta}+(1-t)\bm{\theta}'$ stays in
the domain, so the uniform bound applies along the entire segment; the fundamental theorem
of calculus gives strong monotonicity, and Cauchy--Schwarz gives the reverse Lipschitz
property of the moment map, so $\Psi_M$ is $\gamma^{-1}$-Lipschitz. The $Q$-space inverse
map $\Psi$ is obtained through $q\mapsto x=S_a q$
($S_a=\mathrm{diag}(a_1,\dots,a_K)\otimes I_K=R_a^{-1}$) composed with $\Phi^{\top}$ and
then $\Psi_M$.

\textbf{The monotone-operator viewpoint.} On the quotient space,
$M:\bm{\theta}\mapsto\Phi^{\top}x_{\bm{\theta}}$ is a strongly monotone operator. By the
Zarantonello--Minty theorem \cite{zarantonello1960,minty1962monotone}, a strongly monotone
map is injective with a Lipschitz inverse; \cite{rockafellar1976monotone} gives the
proximal-point algorithmic framework for monotone operators. The paper-specific content is
the uniform lower bound on the modulus $\gamma$ derived from the Sinkhorn linearization
(Section~\ref{sec:sinkhorn_linearization}).

\begin{assumption}[Compact domain for the global statement]\label{asm:pos}
This assumption is needed \textbf{only} for the global statement T3(b); the local
statement T3(a) does not require it. Assume the parameter is restricted to a compact
convex domain $\Theta\subset\mathbb{R}^F/\mathcal{N}_\Phi$ on which the Sinkhorn plan
entries are uniformly bounded away from 0:
\[
\pi_{\min}:=\inf_{\bm{\theta}\in\Theta}\min_{ij}\pi_{ij}(\bm{\theta})>0,
\]
and the gauge-cleaned Gram has a minimal eigenvalue:
$\lambda_{\min}^{\Theta}:=\lambda_{\min}(\Sigma)>0$ (the global version of the pointwise
identifiability condition).
\end{assumption}

\subsection{Numerical verification}

See Section~\ref{sec:exp} for details.

\section{T4: Convergence}\label{sec:t4}

\subsection{Statement}

\begin{theorem}[Monotone convergence to a local minimum]\label{thm:T4}
Under correct specification, let the population cross-entropy objective be
\[
\ell(\bm{\theta})=\sum_{l=1}^{L}w_l\,\mathrm{CE}(Q_{\bm{\theta}^*}^{(l)}\|Q_{\bm{\theta}}(a_l,b_l)),\qquad \sum_l w_l=1,\; w_l\ge0,
\]
where $Q_{\bm{\theta}^*}^{(l)}$ is the conditional transition operator at the true
parameter $\bm{\theta}^*$ (the population objective, not the empirical objective
$\ell_n$). The population Hessian at the true parameter $\bm{\theta}^*$ is positive
definite (derived from the core spectral bound), and gradient descent with a sufficiently
small step size converges monotonically to a local minimum, provided all iterates remain
within the local strong-convexity neighborhood of $\bm{\theta}^*$ (e.g.\ the sublevel set
containing the initial point). For the statistical theory of entropic OT see
\cite{bigot2019central} (central limit theorems), \cite{mena2019statistical} (statistical
bounds), and \cite{pooladian2021entropic} (entropy estimation).
\end{theorem}

\subsection{Proof roadmap}

\textbf{Population cross-entropy and weights.} To cover multiple marginals
simultaneously, let $l=1,\dots,L$ with weights $w_l\ge0$, $\sum_lw_l=1$; for each source
state $i$ let the cross-entropy weight be $\rho_{l,i}>0$. The operator-level objective of
this paper weights every row equally, i.e.\ $\rho_{l,i}=1$; a likelihood weighted by the
joint sample frequencies corresponds to general $\rho_{l,i}$. The population loss is
\[
\ell(\bm{\theta})=-\sum_{l=1}^{L}w_l\sum_{i=1}^{K}\rho_{l,i}\sum_{j=1}^{K}Q_{\bm{\theta}^*}^{(l)}(i,j)\log Q_{\bm{\theta}}^{(l)}(i,j).
\]

\textbf{Exact form of the Hessian (the second-order residual vanishes).} Write
$\mathcal{J}_l(\bm{\theta})=\partial q_{\bm{\theta}}^{(l)}/\partial\bm{\theta}$ and the
diagonal weight at the truth
$W_l=\mathrm{diag}_{i,j}(\rho_{l,i}/Q_{\bm{\theta}^*}^{(l)}(i,j))$. Differentiating the
population loss once gives
$\nabla\ell(\bm{\theta})=-\sum_lw_l\mathcal{J}_l(\bm{\theta})^{\top}(\rho_l\odot q_l^*/q_l(\bm{\theta}))$
(division coordinatewise). Differentiating again, two types of terms appear at
$\bm{\theta}^*$: the Jacobian outer product, and the second-derivative term of $q_l$. For
a fixed source state $i$, the coefficient of the second-order term is
$-w_l\rho_{l,i}\sum_j\partial^2Q_{\bm{\theta}}^{(l)}(i,j)/\partial\bm{\theta}\partial\bm{\theta}^{\top}\big|_{\bm{\theta}^*}$.
Since $\sum_jQ_{\bm{\theta}}^{(l)}(i,j)=1$ holds identically in $\bm{\theta}$, double
differentiation gives
$\sum_j\partial^2Q_{\bm{\theta}}^{(l)}(i,j)/\partial\bm{\theta}\partial\bm{\theta}^{\top}=0$,
so the second-derivative residual term \textbf{vanishes row by row}, with no cross terms.
The population Hessian at the truth is exactly the Fisher-type matrix
\[
\boxed{\nabla^2\ell(\bm{\theta}^*)=\sum_{l=1}^{L}w_l\,\mathcal{J}_l(\bm{\theta}^*)^{\top}W_l\,\mathcal{J}_l(\bm{\theta}^*)}.
\]

\textbf{Curvature lower bound.} From $Q_{\bm{\theta}^*}^{(l)}(i,j)\le1$ we get
$W_l\succeq\rho_{\min}I$, where $\rho_{\min}:=\min_{l,i}\rho_{l,i}$. Combined with the core
spectral bound \eqref{eq:38},
\[
\lambda_{\min}(\nabla^2\ell(\bm{\theta}^*))\ge\sum_lw_l\rho_{\min}\sigma_{\min}(\mathcal{J}_l(\bm{\theta}^*))^2
\ge\sum_lw_l\rho_{\min}\Bigl(\tfrac{\pi_{\min,l}}{a_{l,\max}\varepsilon}\Bigr)^2\lambda_{\min}(\Sigma).
\]

For the operator-level objective, take $\rho_{\min}=1$ and use the looser bound
$a_{\max}\le1$, giving
\[
\lambda_{\min}(\nabla^2\ell(\bm{\theta}^*))\ge\frac{\pi_{\min,*}^2}{\varepsilon^2}\lambda_{\min}(\Sigma)=:\mu>0,\qquad \pi_{\min,*}:=\min_l\pi_{\min,l}.
\]

When the right-hand side is strictly positive, $\ell$ has a positive definite Hessian at
the truth, derived from the core spectral bound. By continuity of the Hessian, there exist
a neighborhood $U$ of the truth and constants $0<\mu_0<\mu$ and $L_\ell<\infty$ such that
$\mu_0 I\preceq\nabla^2\ell(\bm{\theta})\preceq L_\ell I$ ($\bm{\theta}\in U$), so $\ell$
is locally strongly convex on $U$ with Lipschitz gradient. $\mu_0$ is generally strictly
smaller than $\mu$ (by Hessian continuity, $\mu_0=\mu-\delta$ for some $\delta>0$).
Gradient descent $\bm{\theta}_{r+1}=\bm{\theta}_r-\eta\nabla\ell(\bm{\theta}_r)$, when the
initial point lies in the basin of attraction within $U$ and $0<\eta\le1/L_\ell$, satisfies
$\ell(\bm{\theta}_{r+1})-\ell(\bm{\theta}^*)\le(1-\eta\mu_0)(\ell(\bm{\theta}_r)-\ell(\bm{\theta}^*))$,
converging monotonically to the local minimum.

\begin{proposition}[Concentration of the empirical Hessian]\label{prop:hess_conc}
Let $\widehat{\bm{\theta}}_n^{\mathrm{CE}}$ be the empirical minimizer of the normalized
empirical cross-entropy objective, distinct from the $\ell_1$-penalized estimate
$\widehat{\bm{\theta}}_n^{\mathrm{lasso}}$ and the debiased refit estimate
$\widetilde{\bm{\theta}}_{S,n}$ of T2. On a neighborhood $U$ of the true parameter, assume
the full per-sample cross-entropy Hessian contribution $\nabla_{\bm{\theta}}^2\psi(Z;\bm{\theta})$
satisfies
\[
\sup_{\bm{\theta}\in U,z}\|\nabla_{\bm{\theta}}^2\psi(z;\bm{\theta})\|_{\mathrm{op}}\le M_H,
\]
and that the empirical Hessian is Lipschitz in the parameter:
\[
\|\nabla^2\ell_n(\bm{\theta})-\nabla^2\ell_n(\bm{\theta}')\|_{\mathrm{op}}
\le L_H\|\bm{\theta}-\bm{\theta}'\|.
\]
If for some $r_n>0$ one has the high-probability estimator error bound
\[
\mathbb P\left(\|\widehat{\bm{\theta}}_n^{\mathrm{CE}}-\bm{\theta}^*\|>r_n\right)\le\delta_\theta(n),
\]
then for any $t>0$, matrix Hoeffding and the triangle inequality give
\[
\mathbb P\left(
\|\nabla^2\ell_n(\widehat{\bm{\theta}}_n^{\mathrm{CE}})-\nabla^2\ell(\bm{\theta}^*)\|_{\mathrm{op}}>t+L_Hr_n
\right)
\le 2F\exp\left(-\frac{nt^2}{2M_H^2}\right)+\delta_\theta(n).
\]
In particular, if $L_Hr_n\le\mu/4$ and one takes $t=\mu/4$, then whenever
\[
n\ge\frac{32M_H^2}{\mu^2}\log\frac{2F}{\delta_H},
\qquad
\delta_H+\delta_\theta(n)\le\delta,
\]
with probability at least $1-\delta$ one has
\[
\lambda_{\min}\left(\nabla^2\ell_n(\widehat{\bm{\theta}}_n^{\mathrm{CE}})\right)\ge\frac\mu2>0.
\]
Here $\mu=(\pi_{\min,*}^2/\varepsilon^2)\lambda_{\min}(\Sigma)$ is the curvature lower
bound of the population Hessian. The rate $r_n=O(n^{-1/2})$ requires a separate proof of
the corresponding estimator error tail bound.
\end{proposition}

The IOT objective has no known analogue of a condition that rules out spurious local
minima; this paper retains the local-minimum structure.

\subsection{Numerical verification}

See Section~\ref{sec:exp} for details.

\section{O5: Misspecification Analysis}\label{sec:t5}

\subsection{Statement}

\begin{observation}[Projection onto the OT model set]\label{obs:T5}
Let $\mathcal{M}_{\mathrm{OT}}$ be the model set of conditional transition operators
generated by entropic OT, and $Q^{\star}$ the true (possibly non-OT) transition operator.
Under a compact convex parameter domain $\Theta$ (Assumption~\ref{asm:pos}), the image of
$Q_{\bm{\theta}}$ is compact; define the projection residual
$\mathrm{Res}(Q^{\star})=\min_{Q'\in\mathcal{M}_{\mathrm{OT}}}\mathrm{CE}(Q^{\star}\|Q')$
(the minimum is attained). Without compactness, replace $\min$ by $\inf$.

(1) Convergence (standard $M$-estimation). The IOT estimator converges to the pre-image of
the projection of $Q^{\star}$ onto $\mathcal{M}_{\mathrm{OT}}$. If the projection minimizer
is unique and lies in the identifiable region $\mathcal{R}$, then
\[
\widehat{\bm{\theta}}_n\xrightarrow{p}\Psi\Bigl(\arg\min_{Q'\in\mathcal{M}_{\mathrm{OT}}}\mathrm{CE}(Q^{\star}\|Q')\Bigr),
\]
where $\Psi$ is the inverse map of T3 (well-defined and Lipschitz on the identifiable
region $\mathcal{R}$ by T1 and T3). If the projection minimizer is not unique (multiple
minima), the statement can only be phrased as convergence in probability of
$\widehat{\bm{\theta}}_n$ to the set of minimizers in distance; if the minimizer falls
outside $\mathcal{R}$, then $\Psi$ is set-valued (gauge-kernel ambiguity) and the
conclusion is set-valued convergence.

(2) Residual bound (definitional). The limiting residual $\mathrm{Res}(Q^{\star})$ is
nonnegative, and for any random OT model $Q^{\mathrm{rand}}\in\mathcal{M}_{\mathrm{OT}}$
lying on the manifold,
\[
\mathrm{Res}(Q^{\star})\le\mathrm{CE}(Q^{\star}\|Q^{\mathrm{rand}}),
\]
a bound that is the \textbf{definitional minimality} of the projection: $\mathrm{Res}$ does
not exceed the value at any feasible point.
\end{observation}

\begin{remark}[Dependence of O5 on T1--T4 and the H\"older exponent]
O5 depends on T1 (the identifiable region) and T3 (the Lipschitz inverse map), and not on
T2 or T4. The H\"older continuity of the projection map
$Q\mapsto\arg\min_{Q'\in\mathcal{M}_{\mathrm{OT}}}\mathrm{CE}(Q\|Q')$ controls the
robustness of the estimator under misspecification. The experiments
(Section~\ref{sec:exp}) estimate the empirical effective exponent
$\alpha_{\mathrm{eff}}\in(0,1)$ via log-log regression across multiple settings; the
$\varepsilon=0.1$ setting is affected by Adam optimization residual, and the reliability of
$\alpha_{\mathrm{eff}}$ should be assessed jointly with the condition number analysis (T1)
and optimization convergence (T4).
\end{remark}

\subsection{Proof roadmap (consistency)}

\textbf{The pseudo-true projection.} Let the true conditional transition operator be
$Q^{\star}$ (not necessarily from the OT model of this paper), the model manifold
$\mathcal{M}_{\mathrm{OT}}=\{Q_{\bm{\theta}}:\bm{\theta}\in\Theta\}$, and the population
conditional cross-entropy
$\mathcal{L}_{Q^{\star}}(\bm{\theta})=\mathrm{CE}(Q^{\star}\|Q_{\bm{\theta}})=-\sum_i\rho_i\sum_jQ^{\star}(i,j)\log Q_{\bm{\theta}}(i,j)$
($\rho_i>0$ the source-state weights). The operator-level objective of this paper takes
$\rho_i=1$ (equal row weights), consistent with $\rho_{l,i}=1$ in the T4 population
objective above; if the samples are drawn at the joint frequencies, $\rho_i$ can be taken
as the empirical marginal frequency $a_i$ of the source states. If
$Q^{\star}\in\mathcal{M}_{\mathrm{OT}}$ and T1 holds, the projection is
$Q^{\dagger}=Q^{\star}$; otherwise $Q^{\dagger}$ is the pseudo-true model point.

\textbf{The consistency argument.} First make explicit the empirical process to which the
empirical objective corresponds. In this paper $\widehat Q$ is the row-normalized
conditional operator and the empirical objective uses equal row weights ($\rho_i=1$), so
the ULLN is stated for this operator-level empirical process
$\mathcal{L}_{n,Q^{\star}}(\bm{\theta}):=-\sum_i\sum_j\widehat Q_n(i,j)\log Q_{\bm{\theta}}(i,j)$;
if the samples are drawn at the joint frequencies and weighted by $\rho_i=a_i$, the
empirical process is different, and the corresponding weighted ULLN must be proved
separately---the equal-row-weight form here cannot be applied directly. Assume $\Theta$ is
compact, $Q_{\bm{\theta}}(i,j)$ is continuous and uniformly positive on $\Theta$, and the
population objective has a unique minimizer
$\bm{\theta}^{\dagger}\in\arg\min_{\bm{\theta}\in\Theta}\mathcal{L}_{Q^{\star}}(\bm{\theta})$.
Then the function family $\{-\log Q_{\bm{\theta}}(i,j):\bm{\theta}\in\Theta\}$ is uniformly
bounded and equicontinuous on the finite state space, and the operator-level empirical
objective $\mathcal{L}_{n,Q^{\star}}$ satisfies the uniform law of large numbers
\[
\sup_{\bm{\theta}\in\Theta}\bigl|\mathcal{L}_{n,Q^{\star}}(\bm{\theta})-\mathcal{L}_{Q^{\star}}(\bm{\theta})\bigr|\xrightarrow{p}0.
\]

Take any convergent subsequence $\widehat{\bm{\theta}}_{n_r}$; its limit
$\bar{\bm{\theta}}$ satisfies
$\mathcal{L}_{Q^{\star}}(\bar{\bm{\theta}})\le\mathcal{L}_{Q^{\star}}(\bm{\theta})$
($\bm{\theta}\in\Theta$) by empirical minimality and the ULLN; unique minimization gives
$\bar{\bm{\theta}}=\bm{\theta}^{\dagger}$. All convergent subsequences have the same limit,
so $\widehat{\bm{\theta}}_n\xrightarrow{p}\bm{\theta}^{\dagger}$ and
$Q_{\widehat{\bm{\theta}}_n}\xrightarrow{p}Q^{\dagger}$. If the projection is not unique
(the population objective has multiple global minima), the statement can only be phrased as
convergence of the distance from the estimator to the set of minimizers
$\mathcal{S}^{\dagger}:=\arg\min_{\bm{\theta}\in\Theta}\mathcal{L}_{Q^{\star}}(\bm{\theta})$:
$\mathrm{dist}(\widehat{\bm{\theta}}_n,\mathcal{S}^{\dagger})\xrightarrow{p}0$. If
$\widehat{\bm{\theta}}_n$ and $\bm{\theta}^{\dagger}$ both lie in the same compact convex
regular domain $\Theta_{\mathrm{reg}}$ and $Q^{\dagger}$ lies in its model image, T3 gives
the local stability transfer
$\|\widehat{\bm{\theta}}_n-\bm{\theta}^{\dagger}\|\le L_\Theta\|Q_{\widehat{\bm{\theta}}_n}-Q^{\dagger}\|$.

\subsection{Numerical verification}

See Section~\ref{sec:exp} for details.

\section{Experiments}\label{sec:exp}

\paragraph{Protocol.}
All experiments fix the numpy and PyTorch random seeds for reproducibility; the main
setting is $\varepsilon=0.2$, $K\in\{4,5,6\}$, and the number of features
$F\in\{6,8,10,15\}$, with E6/E7/E10 additionally scanning $\varepsilon$; marginals are
generated randomly and the sample size ranges from 1000 to 40000; Sinkhorn plans are
computed in double precision. The self-contained implementation includes a differentiable
Sinkhorn, the feature-parameterized cost, Adam with cosine annealing, gradient clipping,
and non-finite restart. The aggregated results are in
\texttt{09\_iot\_theory/output/experiments\_e1\_e10.json}.

\subsection{T1: Identifiability (E6, E7)}

\begin{table}[tbhp]
\centering
\begin{tabular}{lccc}
\toprule
Setting & Rank & Condition number & Direction recovery correlation \\
\midrule
Full rank & 6 & 2.58 & 0.999 \\
Collinear (feature 2 $\approx$ feature 1) & 6 & 2474 & 0.958 \\
Pure gauge feature & 5 & --- & not recoverable \\
\bottomrule
\end{tabular}
\caption{Comparison of identifiability settings}
\label{tab:identifiability}
\end{table}

Verification of the marginal lower bound: $K{=}4,F{=}6$ already reaches full rank $6$ at
$L=1$; $K{=}3,F{=}5$ is capped at $(K-1)^2=4$ and $K{=}4,F{=}15$ is capped at $8$, and
adding $L$ has no effect; for $K{=}5,F{=}10$ the condition number improves with $L$ from
numerical ill-posedness (with a single marginal $\mathcal{S}_1$ is nearly singular, with
condition number $\approx10^{16}$, recorded as $\infty$) down to $2582$---an improvement in
numerical well-posedness, not a repair of the rank (T1 has proved that the rank is
independent of $L$). Collinearity damages the condition number and unique identifiability
(cond $2.58\to2474$), while the direction remains recoverable (corr $0.958$); the pure
gauge feature is the only true failure point.

\paragraph{Condition-number divergence as $\varepsilon\to0$ (E6, 10 seeds).}
For $\varepsilon\in\{0.5,0.3,0.2,0.1,0.05,0.02\}$ the log means are
$\log\mathrm{cond}(\mathcal{S}_L)\in\{5.19,\,6.83,\,3.53,\,4.82,\,10.77,\,15.75\}$, rising
by about $10.6$ log units (a factor of about $3.8\times10^4$) from $\varepsilon=0.5$ to
$\varepsilon=0.02$; the small-$\varepsilon$ tail shows a divergent trend for the given
features and marginals. The middle
segment ($0.2,0.1$) fluctuates because random features occasionally come close to the gauge
directions, showing that the actual condition number need not be strictly monotone. This
corresponds to the degeneration of $\pi_{\min}\to0$ in the core spectral bound
\eqref{eq:38}; the trend is an empirical observation for the given settings.

\begin{figure}[tbhp]
\centering
\includegraphics[width=0.75\textwidth]{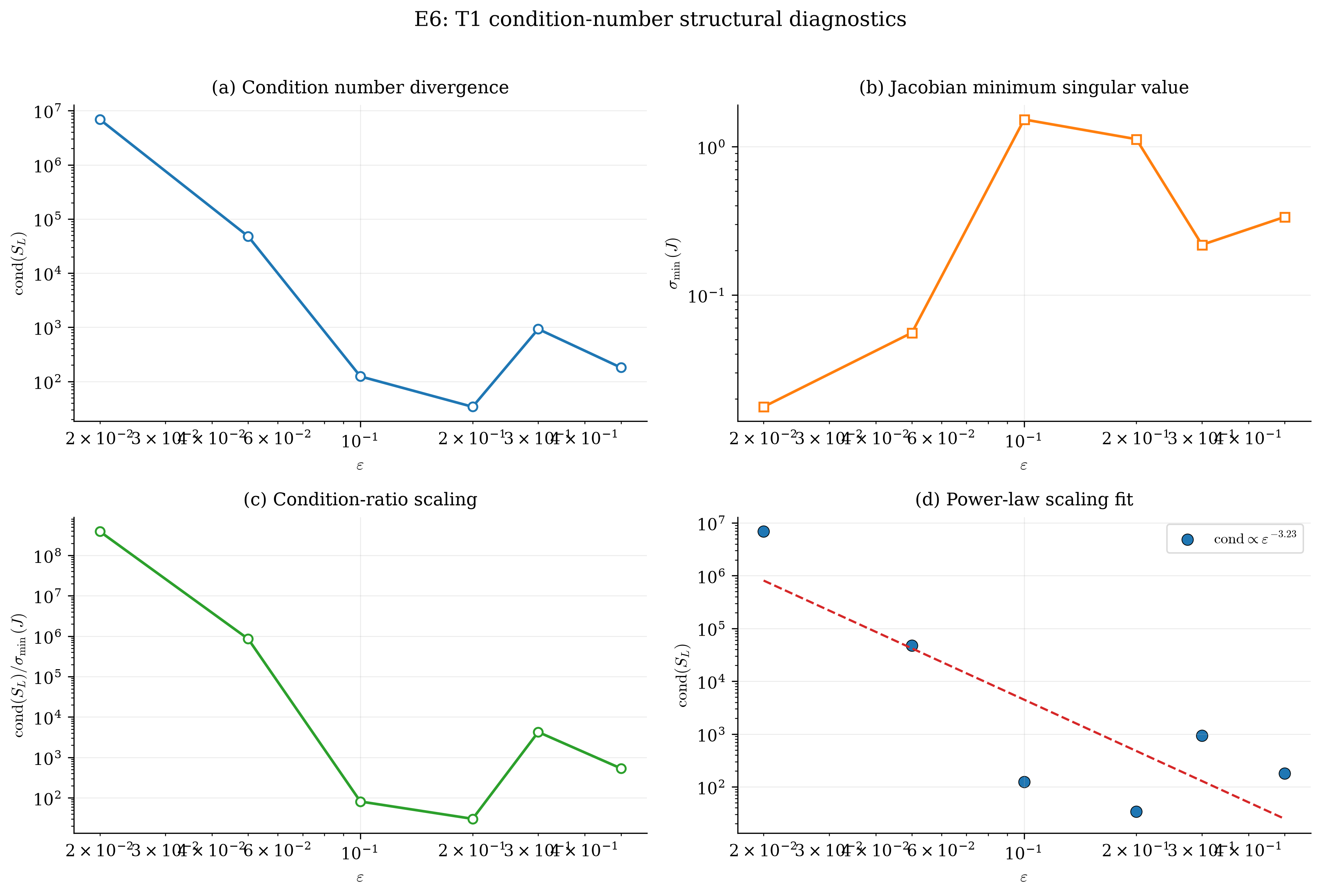}
\caption{Condition-number divergence as $\varepsilon\to0$}
\label{fig:e6_cond}
\end{figure}

\paragraph{Generality (E7).}
The divergence trend is reproduced under
$(K,F,\varepsilon)\in\{(6,8,0.2),(4,6,0.2),(8,12,0.2),(6,8,0.1),(6,8,0.3)\}$ (from
$\varepsilon=0.5$ to $0.05$, cond rises $699\to5.8\times10^7$,
$3.3\times10^4\to1.2\times10^5$, $478\to9.5\times10^4$, $444\to5.4\times10^4$, and
$342\to1.3\times10^7$ respectively).

\begin{figure}[tbhp]
\centering
\includegraphics[width=0.75\textwidth]{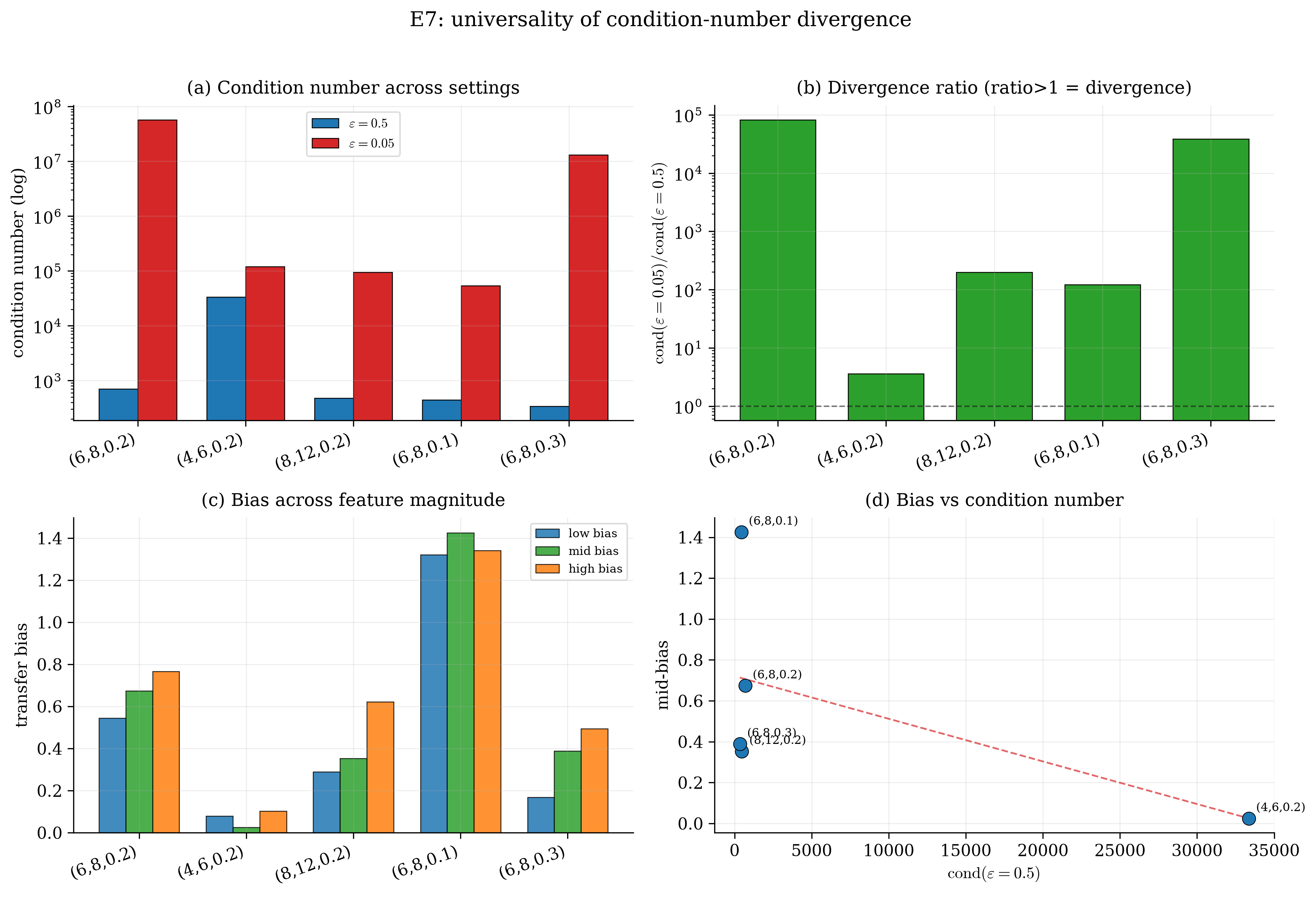}
\caption{Generality verification}
\label{fig:e7_generality}
\end{figure}

\FloatBarrier

\subsection{T2: Sparsistency (E1, E3, E8, E9)}

\paragraph{Recovery curve (E1).}
The support recovery probability is estimated from $100$ seeds per sample size over
$n\in\{200,500,800,1200,2000,3000,5000,8000,12000,20000\}$, with binomial 95\% confidence
intervals.
Figure~\ref{fig:sparsistency} shows a four-panel diagnostic: (a) sample-complexity phase
transition, with recovery probability rising from $0.08$ at $n=200$ to $0.93$ at
$n=8000$, with a slight dip to $0.90$ at $n=12000$ before reaching $0.96$ at $n=20000$;
(b) concentration mechanism decomposition, where the score-event
probability $\|G_n\|_\infty\le t_n$ rises from $0.19$ to $1.00$ and the row-count event
probability from $0.89$ to $1.00$; (c) failure-probability rate diagnostic, with the
linear fit of $-\log\widehat P_{\mathrm{fail}}$ against $n t_n^2/A_{\max}$ verifying the
exponential decay rate $C_2$ of Theorem~\ref{thm:T2}; (d) error decomposition, where the
false-positive rate drops from $0.26$ to $0.008$ and the false-negative rate from $0.05$
to $0$, confirming that support failure is primarily driven by false positives.

\begin{figure}[tbhp]
\centering
\includegraphics[width=0.95\textwidth]{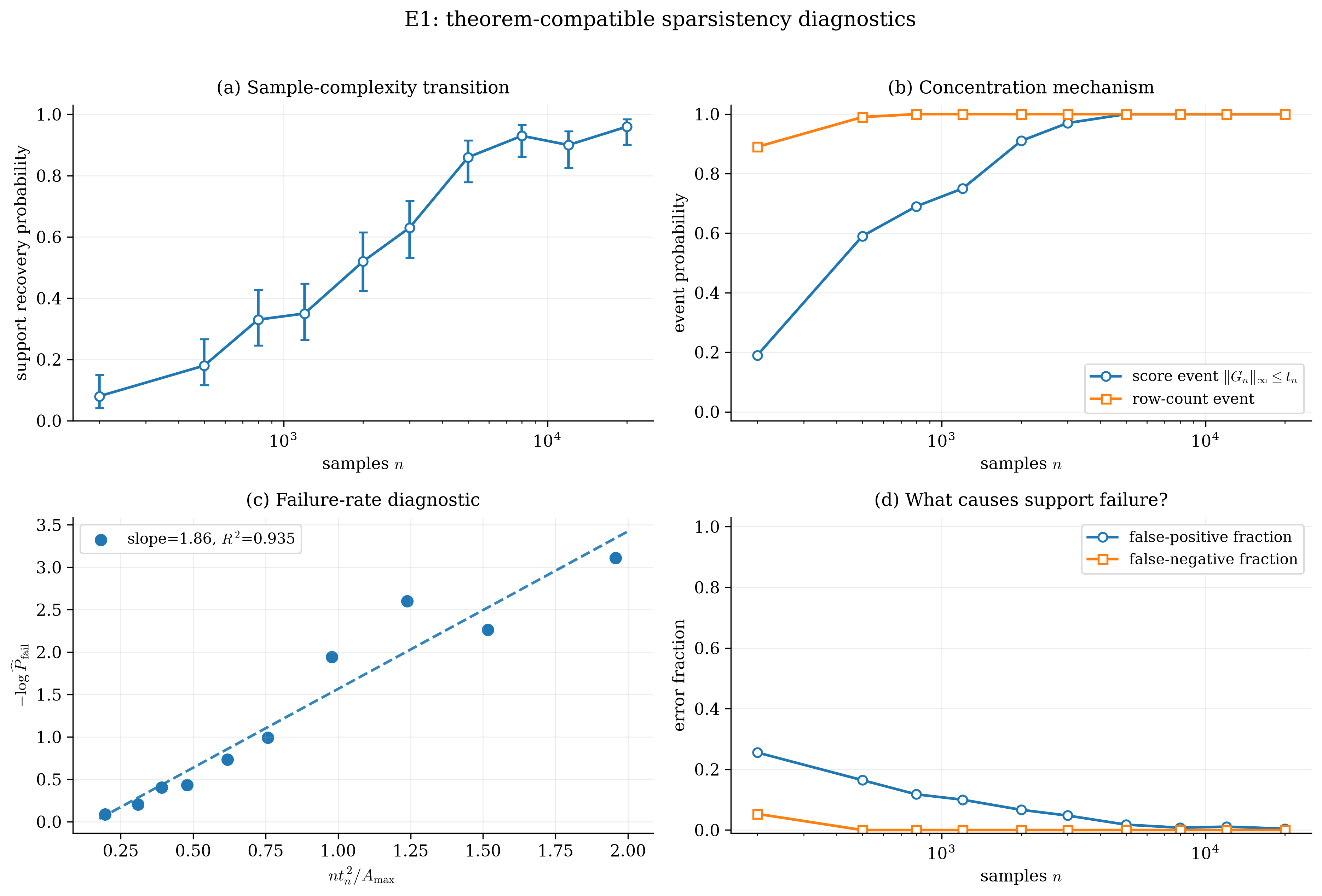}
\caption{Theorem-compatible sparsistency diagnostics: (a) sample-complexity phase
transition; (b) concentration mechanism decomposition; (c) failure-probability rate
diagnostic; (d) decomposition of support failure causes}
\label{fig:sparsistency}
\end{figure}

\paragraph{Voting (E3, 100 seeds).}
Each restart is run on a different bootstrap subsample to decorrelate the failures; with
vote counts $R\in\{1,3,5,7,9\}$ the recovery probabilities are
$\{0.20,0.40,0.30,0.60,0.50\}$, and the single-restart success rate is $0.47$. Voting
raises the recovery rate from $0.47$ to $0.50$--$0.60$ with fluctuations; voting is a conditional stabilizer, the gain being limited by a residual systematic failure component.

\begin{figure}[tbhp]
\centering
\includegraphics[width=0.75\textwidth]{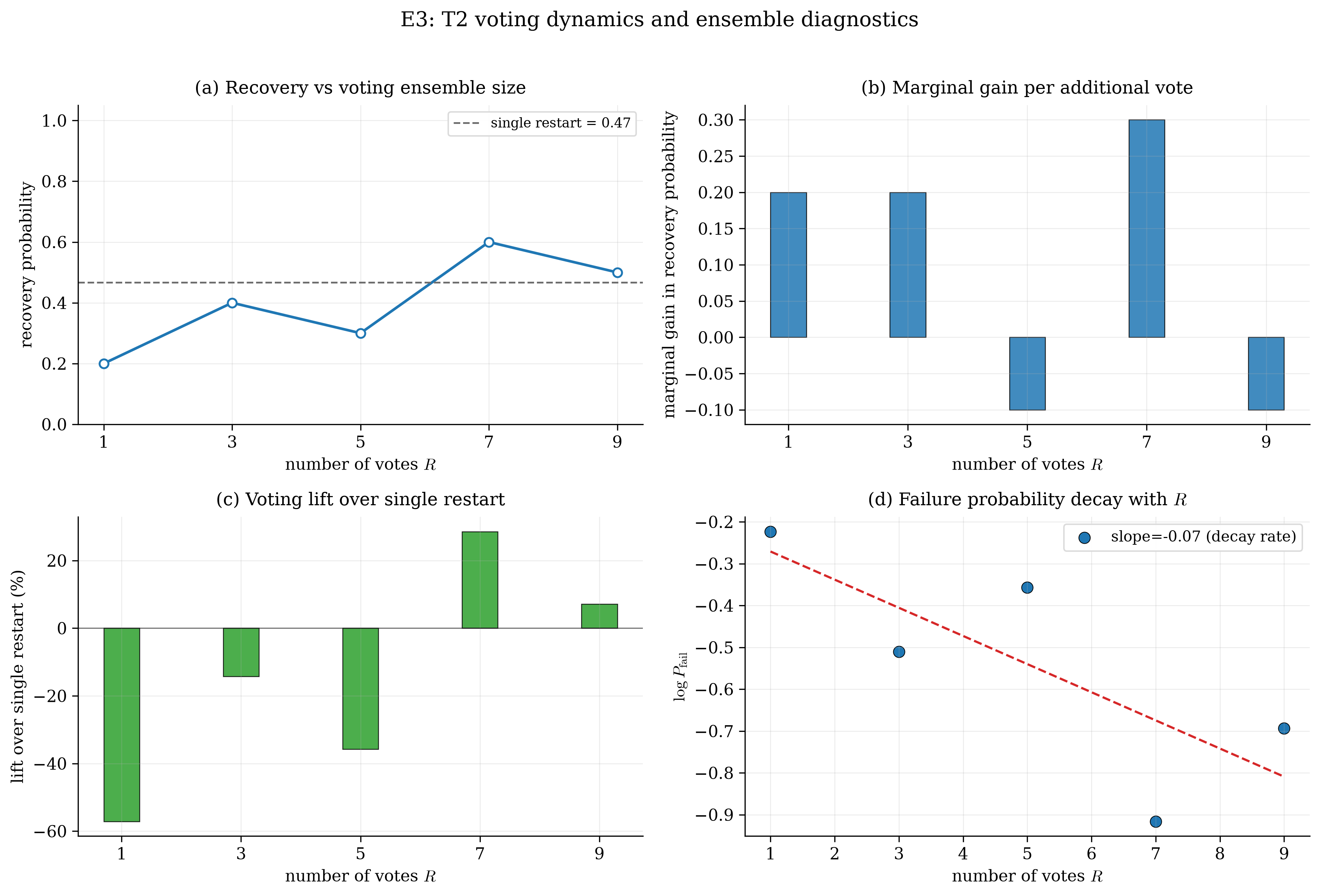}
\caption{Effect of the voting mechanism on the recovery probability}
\label{fig:e3_voting}
\end{figure}

\paragraph{Numerical verification of irrepresentability (E8b, 40 seeds).}
The irrepresentability quantity is computed on the population OT information matrix
$H=\nabla^2\ell(\theta^*)$ (averaged over $L=8$ marginal pairs), landing in $[1.36,7.43]$
with mean $3.19$; the condition $<1$ is satisfied $0/40$ times. Although IR is not
satisfied, recovery still succeeds in E1b, so the condition is positioned as a strictly
sufficient one: it is generically violated on unstructured features, while the recovery
conclusion of the theorem still holds empirically. Lemma~\ref{lem:mutual_incoherence} gives
a sufficient condition verifiable a priori from the feature structure, but that condition
is stronger than directly checking $H$-IR; in the experimental setting of this paper, the
sufficient condition of the lemma is not met, and the directly computed $H$-IR condition
fails as well; consequently, the lemma does not apply in this setting. The failure of the
sufficient condition while recovery still succeeds indicates that the condition, though
sufficient, is not necessary in this setting.

\begin{figure}[tbhp]
\centering
\includegraphics[width=0.75\textwidth]{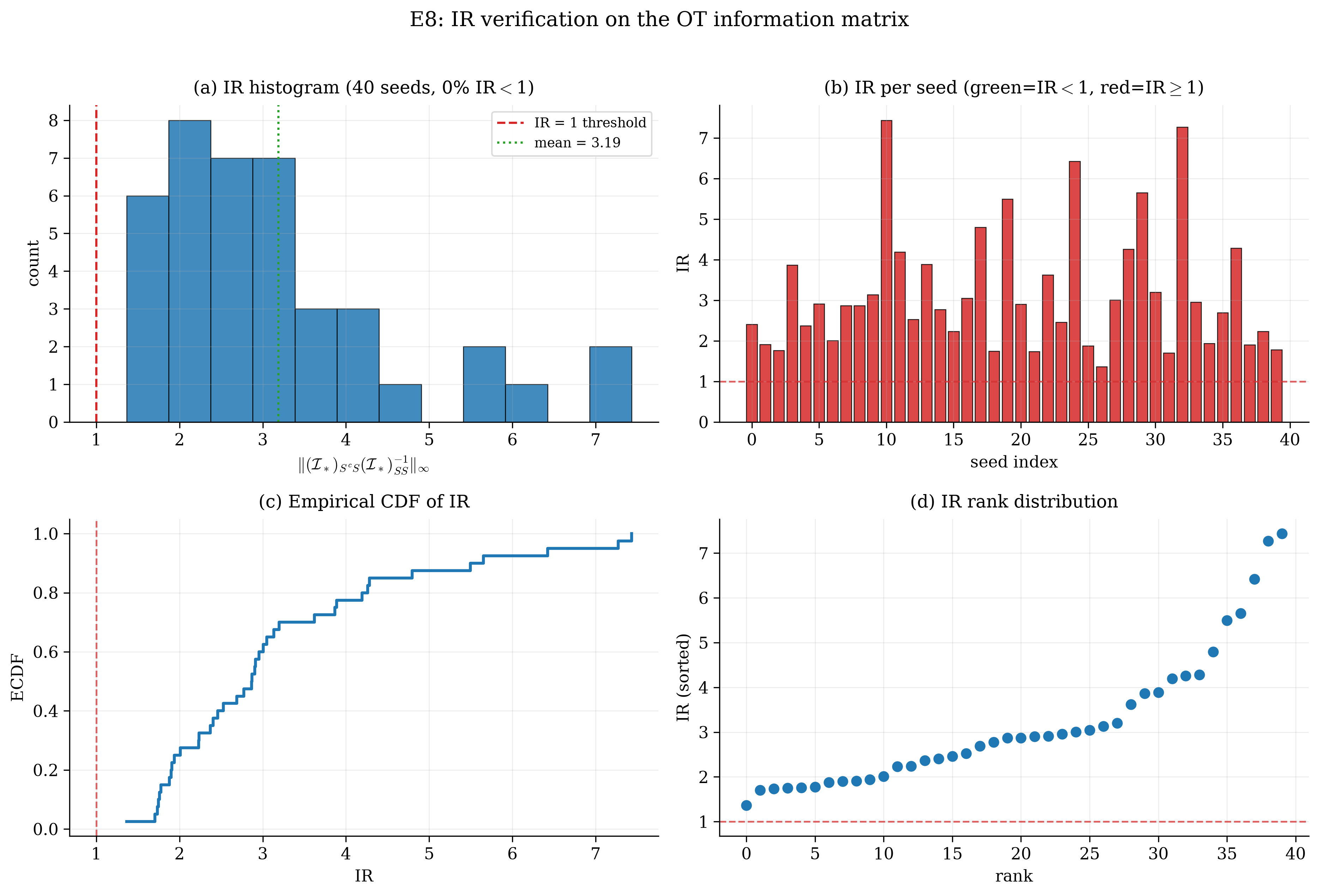}
\caption{Numerical verification of the irrepresentability condition}
\label{fig:e8_ir}
\end{figure}

\paragraph{Comparison of Bernstein constants (E9, 8 seeds).}
The ratio of the Hoeffding constant $C_2=2/\Delta_{\max}^2$ to the Bernstein refinement
$C_2^{\mathrm{B}}$ lies in $[13.7,33.9]$ (mean $24.4$): the variance-aware Bernstein
constant is $1$--$2$ orders of magnitude larger than the range-based Hoeffding constant,
confirming that the refinement is substantially sharper.

\begin{figure}[tbhp]
\centering
\includegraphics[width=0.75\textwidth]{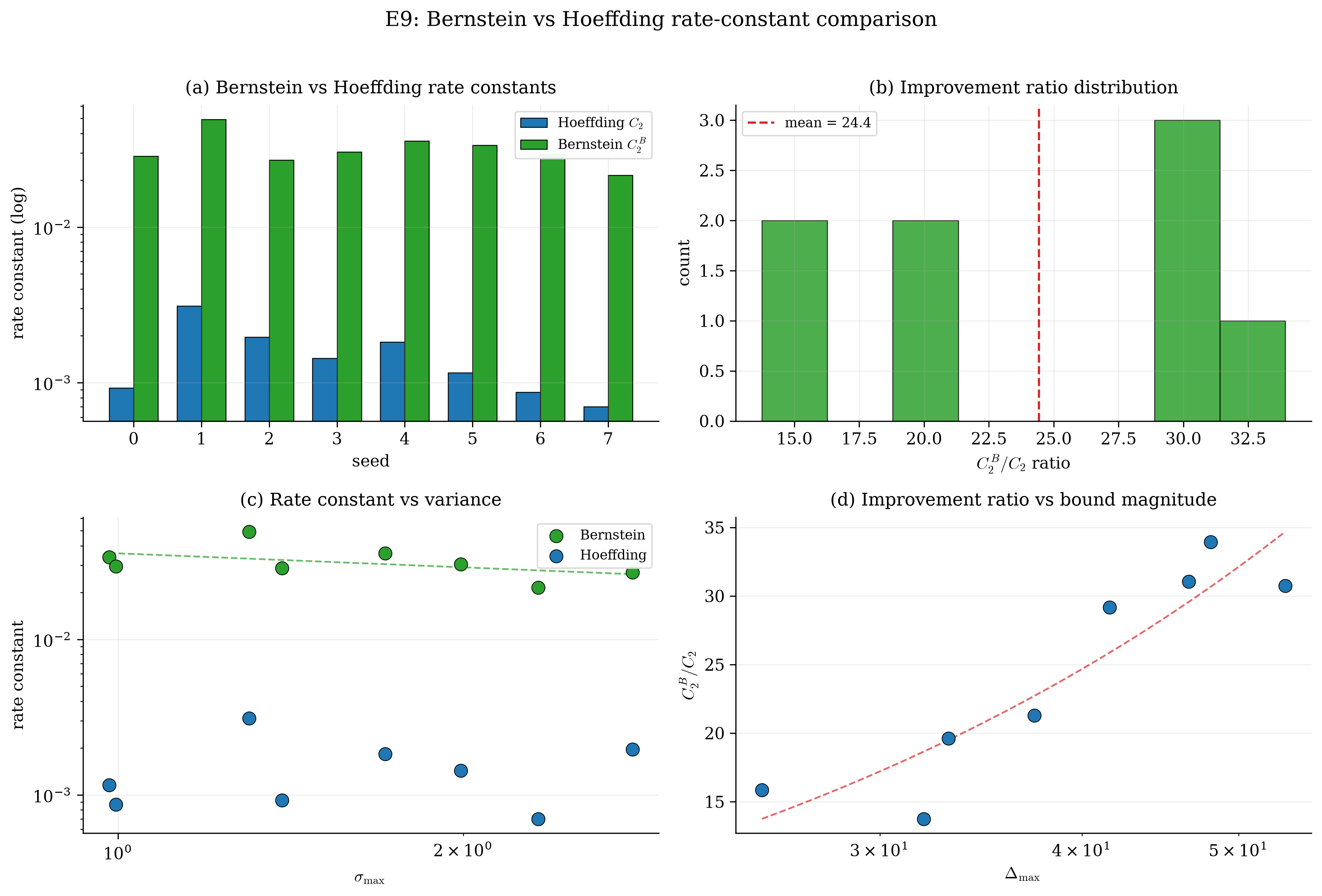}
\caption{Comparison of the Bernstein and Hoeffding constants}
\label{fig:e9_bernstein}
\end{figure}

\FloatBarrier

\subsection{T3: Well-posedness (E2)}

\begin{table}[tbhp]
\centering
\begin{tabular}{lc}
\toprule
Perturbation $\delta Q$ & Transfer error \\
\midrule
0.0003 & 0.0018 \\
0.001  & 0.0030 \\
0.003  & 0.0067 \\
0.01   & 0.0076 \\
\bottomrule
\end{tabular}
\caption{Perturbation transfer error}
\label{tab:perturbation}
\end{table}

\paragraph{The $\varepsilon'$-bias V-shape (E2, 40 seeds).}
The transfer bias is averaged over $40$ seeds with mean $\pm$ standard error reported; on
the grid $\varepsilon'\in\{0.05,0.1,0.2,0.3,0.5\}$ it is
$0.742\pm0.037,\,0.718\pm0.046,\,0.663\pm0.020,\,0.724\pm0.032,\,0.819\pm0.047$.
The quadratic fit has significant positive curvature ($p=0.024$, $R^2=0.77$), so the
V-shape is statistically supported; the empirical minimum is at $\varepsilon'\approx0.2$,
matching the generating $\varepsilon=0.2$. In the experiments of this paper, the V-shape is an empirical observation that provides finite-sample support for the local minimality of the
squared bias. The original Frobenius-norm bias $B(\varepsilon')$ is theoretically
non-differentiable at the truth (the V-shaped cusp), and the V-shape observed in the
experiment itself supports this fact; the squared bias removes the cusp issue.

\begin{figure}[tbhp]
\centering
\includegraphics[width=0.75\textwidth]{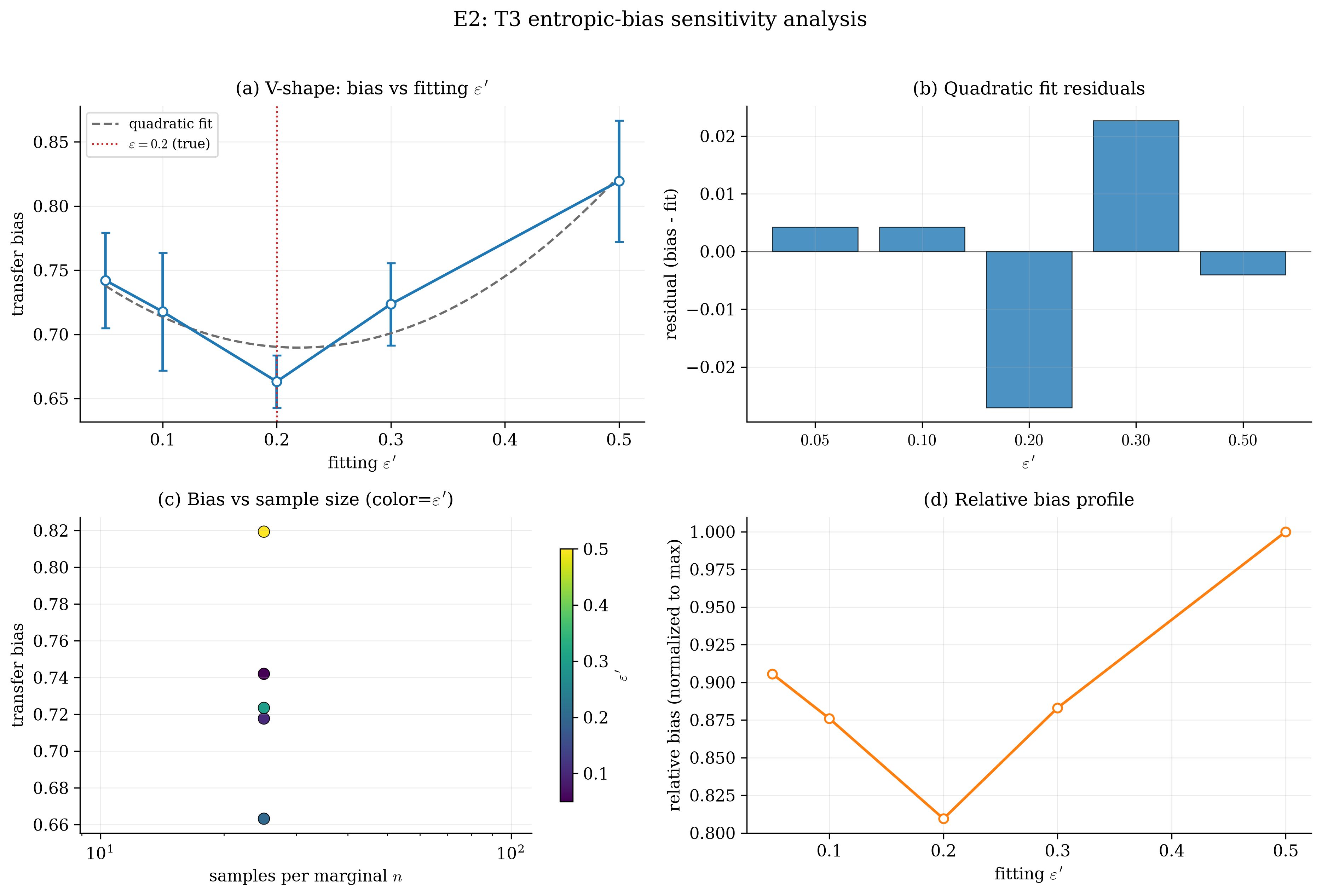}
\caption{The $\varepsilon'$-bias V-shape}
\label{fig:e2_epsbias}
\end{figure}

\FloatBarrier

\subsection{T4: Convergence (E4)}

\paragraph{Comparison of initialization schemes (E4, 40 restarts).}
The success rates of the three schemes---random / multiscale (coarse warm-start) /
prior-guided (initialized on the true support)---are $0.25$, $\textbf{0.95}$, $0.20$
respectively, with mean cross-entropies $14.39$, $14.34$, $15.03$. The multiscale coarse
warm-start raises the success rate from $0.25$ to $0.95$ (about $3.8\times$) with the
lowest objective: local strong convexity (Theorem~\ref{thm:T4}) guarantees that
fixed-step-size gradient descent decreases monotonically within a basin, and the multiscale
scheme raises the fraction of initializations landing in the basin containing the global
optimum. The actual optimization in this paper uses Adam, cosine annealing, and gradient
clipping; their convergence is an empirical phenomenon---Theorem~\ref{thm:T4} addresses
fixed-step-size gradient descent and does not cover adaptive optimizers such as Adam. The
prior-guided scheme performs worst, because fixing the true support injects the correct
active set but biases the magnitudes during optimization. Failures under random
initialization are systematic: a restart stuck in a spurious basin does not escape by
switching to a different random initialization.

\begin{figure}[tbhp]
\centering
\includegraphics[width=0.75\textwidth]{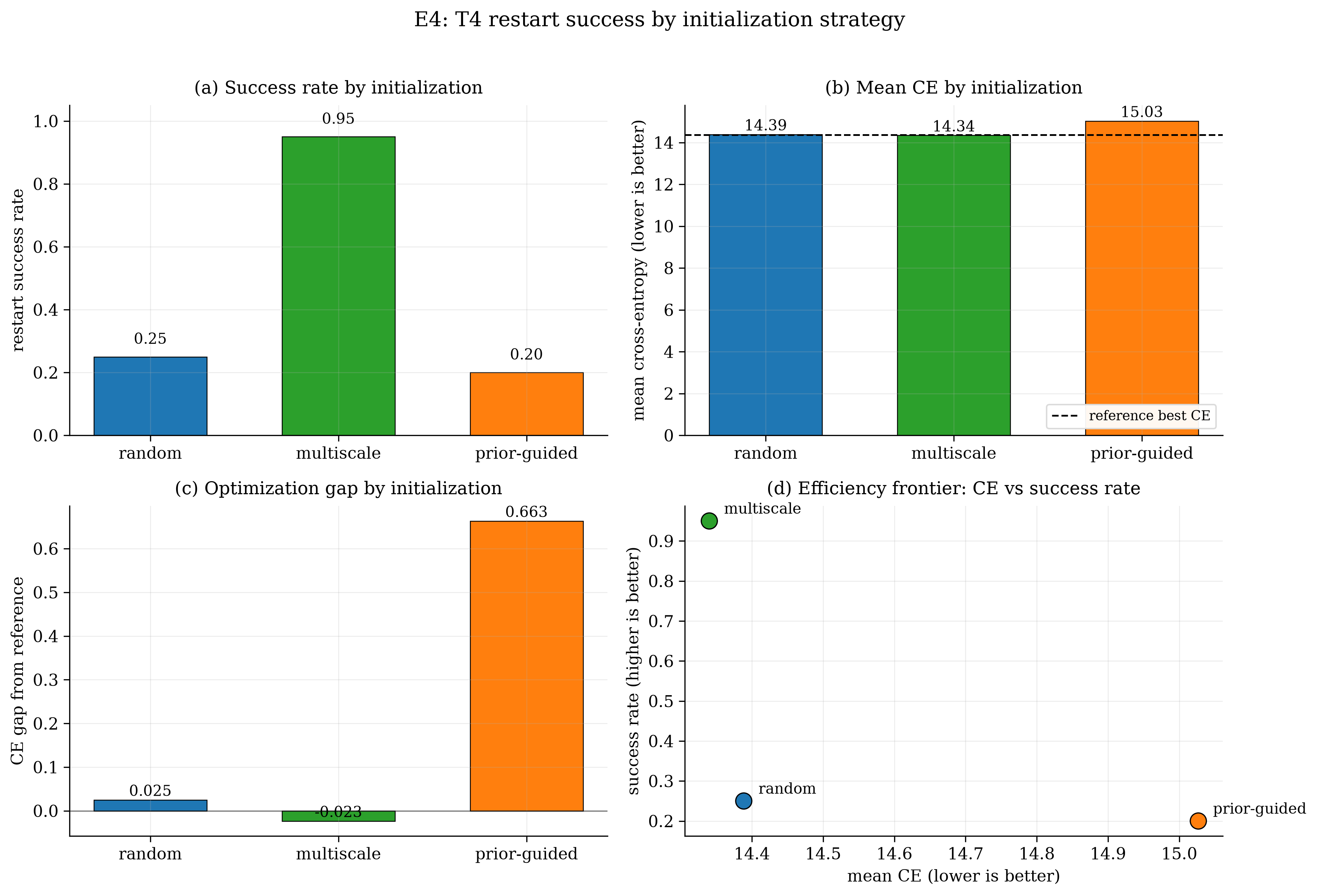}
\caption{Comparison of initialization schemes}
\label{fig:e4_init}
\end{figure}

\FloatBarrier

\subsection{O5: Misspecification (E5)}

Under a non-OT generating mechanism (row softmax plus noise), the projection residual is
$4.50$, larger than the true-OT residual $2.43$ but much smaller than the random-model
residual $10.63$ (ratio $0.42$); the true-OT residual $2.43$ is the total cross-entropy
over all source states (about $0.4$ nats per state), which under correct specification
should approach the entropy $H(Q^*)$ of the true distribution, with the gap to
$H(Q^*)$ reflecting finite-sample estimation error. The fitted
$\bm{\theta}$ collapses toward $0$ (magnitude $0.03$--$0.09$), indicating that the data
cannot be explained by an OT cost.

\paragraph{H\"older exponent across settings (E5).}
The H\"older exponent $\alpha$ of the projection map is estimated by a log--log regression
of the fitted-parameter displacement against the observation perturbation, reproduced under
four $(K,F,\varepsilon)$ settings: $\alpha_{\mathrm{eff}}\in\{0.353,0.266,0.369,0.298\}$
for $(6,8,0.2),(5,6,0.2),(8,12,0.3),(6,8,0.1)$ respectively (per-setting
$R^2\in[0.53,0.92]$, $p\in[4\times10^{-5},0.025]$), all lying in $(0,1)$ and away from the
degenerate endpoint $1$. $\alpha_{\mathrm{eff}}$ is a setting-dependent empirical exponent
estimate; the cross-setting variation reflects the influence of $\varepsilon$ and problem
dimension on the condition number of the projection map.

The $\varepsilon=0.1$ setting yields $\alpha_{\mathrm{eff}}=0.298$, the lowest among the
four. A quantitative analysis reveals that this setting is contaminated by the Adam
optimization residual: comparing the projection results after 500 iterations (the
experimental setting) and 2000 iterations (closer to the true minimum), the optimization
residual is $\|\bm{\theta}_{500}-\bm{\theta}_{2000}\|=0.267$, which exceeds the true
parameter displacement at the three smallest perturbation amplitudes
($10^{-5},3\times10^{-5},10^{-4}$)---$0.107$, $0.093$, and $0.248$, with residual ratios
$2.51$, $2.87$, and $1.08$, respectively. After removing these three residual-dominated
points, the log-log fit over the remaining six signal points yields $\alpha=0.048$
($p=0.67$, $R^2=0.05$), losing statistical significance. At $\varepsilon=0.1$, the elevated
condition number (T1) drives the projection map into an ill-conditioned regime; although the
CE has flattened after 500 iterations ($\Delta\mathrm{CE}\approx10^{-5}$), the parameters
continue to drift, and the optimization residual contaminates the log-log slope estimate at
small perturbations. In the other three settings ($\varepsilon\ge0.2$), the optimization
residual is only $0.003$--$0.016$, and all perturbation responses are well above the
residual (ratios $<0.42$), so the $\alpha$ estimates are reliable.

\begin{figure}[tbhp]
\centering
\includegraphics[width=0.75\textwidth]{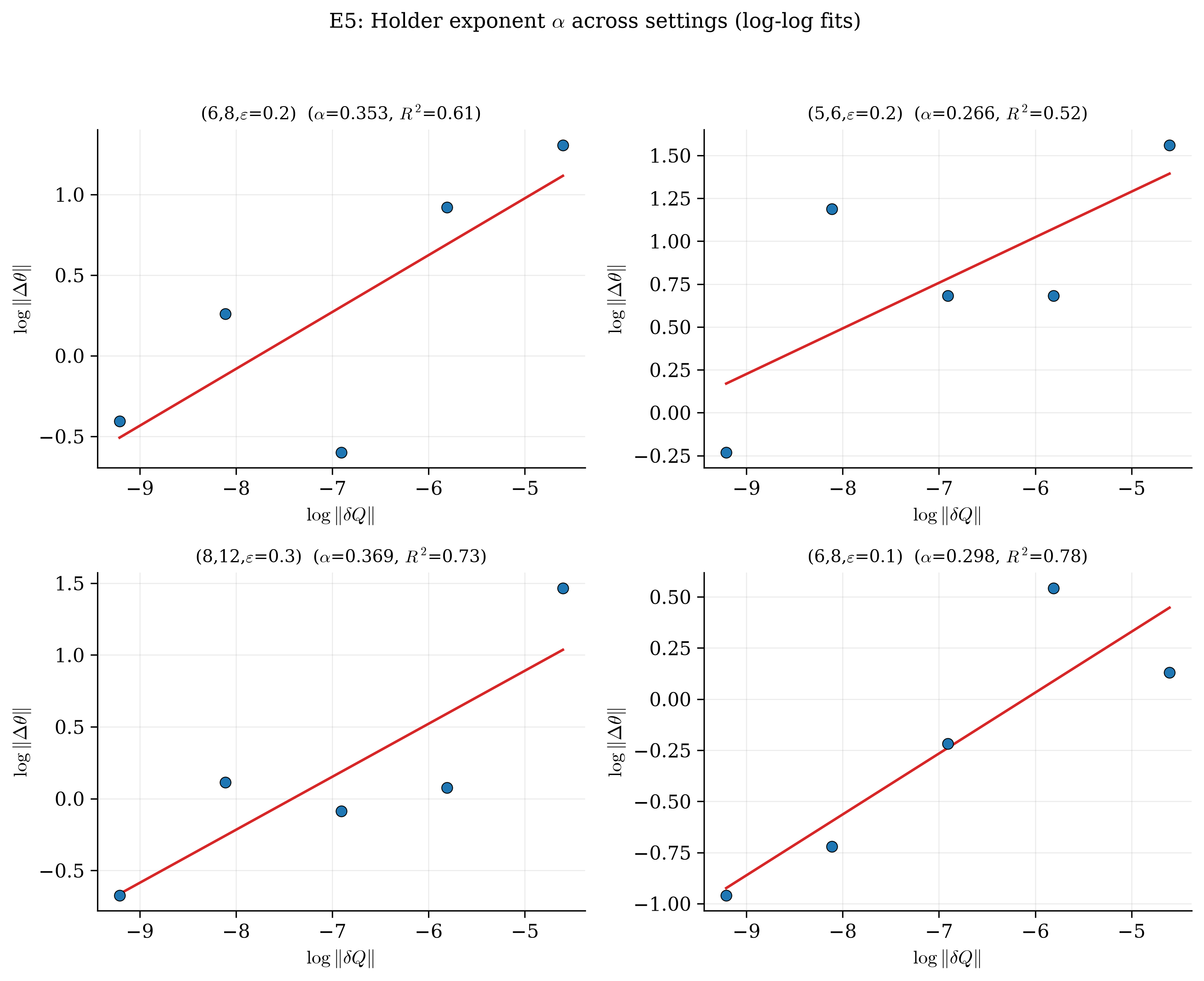}
\caption{Verification of the H\"older exponent across settings}
\label{fig:e5_holder}
\end{figure}

\begin{figure}[tbhp]
\centering
\includegraphics[width=0.75\textwidth]{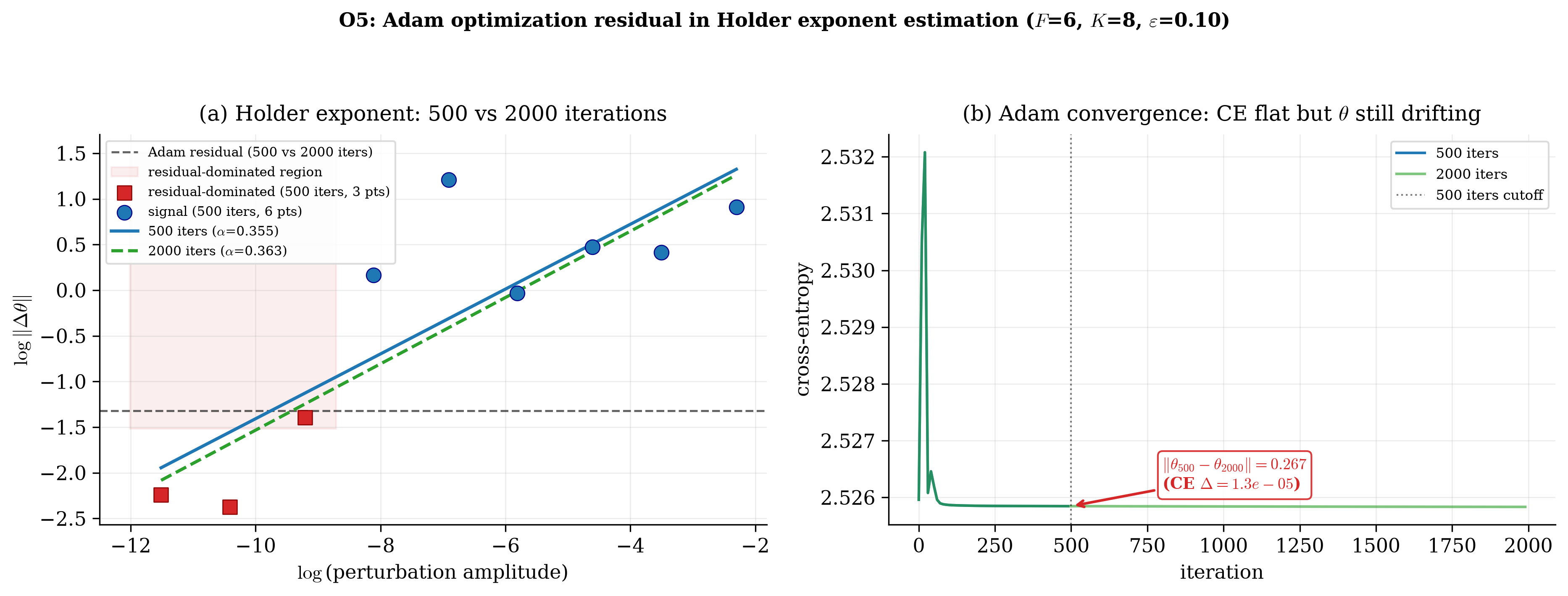}
\caption{Adam optimization residual analysis at $\varepsilon=0.1$: the CE has flattened
after 500 iterations while the parameters continue to drift; the responses at the three
smallest perturbation amplitudes are submerged by the optimization residual}
\label{fig:o5_adam}
\end{figure}

\FloatBarrier

\subsection{The \texorpdfstring{$\varepsilon$}{epsilon}-scaling law of
\texorpdfstring{$\pi_{\min}$}{pi-min} (E10)}

The minimum plan entry $\pi_{\min}(\varepsilon)$ is measured at
$\varepsilon\in\{0.05,0.1,0.2,0.3,0.5\}$, averaged over random marginals in log space.
An empirical fit of $\pi_{\min}\ge c\varepsilon^{\alpha}$ over the finite window
(log--log regression) gives the effective exponent $\alpha_{\mathrm{eff}}\approx19.9$
($R^2=0.89$, $p=0.015$).
These exponents are empirical fits over a finite window, not global polynomial lower
bounds; exponential decay $e^{-\Delta/\varepsilon}$ can also produce a large effective
power exponent over a finite window, and the observed
$\alpha_{\mathrm{eff}}\approx19.9$ may well be exactly this phenomenon. In the unstructured
case $\pi_{\min}$ collapses across many orders of magnitude (from $5.8\times10^{-13}$ at
$\varepsilon=0.5$ to $2.3\times10^{-31}$ at $\varepsilon=0.05$). The large effective exponent
$\alpha_{\mathrm{eff}}\approx19.9$ places the random features in the severely degenerate
regime. The associated Lipschitz constant $L\le\varepsilon/(\pi_{\min}\lambda_{\min}(\Sigma))$
inflates sharply in the degenerate regime; the effective information scale is
$\pi_{\min}(\theta,\varepsilon)\sqrt{\lambda_{\min}(\Sigma)}/\varepsilon$, and a threshold
cannot be given in terms of $\mathrm{cond}(\Sigma)^{-1}$ alone. The statement $\varepsilon\gg\mathrm{cond}(\Sigma)^{-1}$ in this paper is an empirical heuristic.

\begin{figure}[tbhp]
\centering
\includegraphics[width=0.75\textwidth]{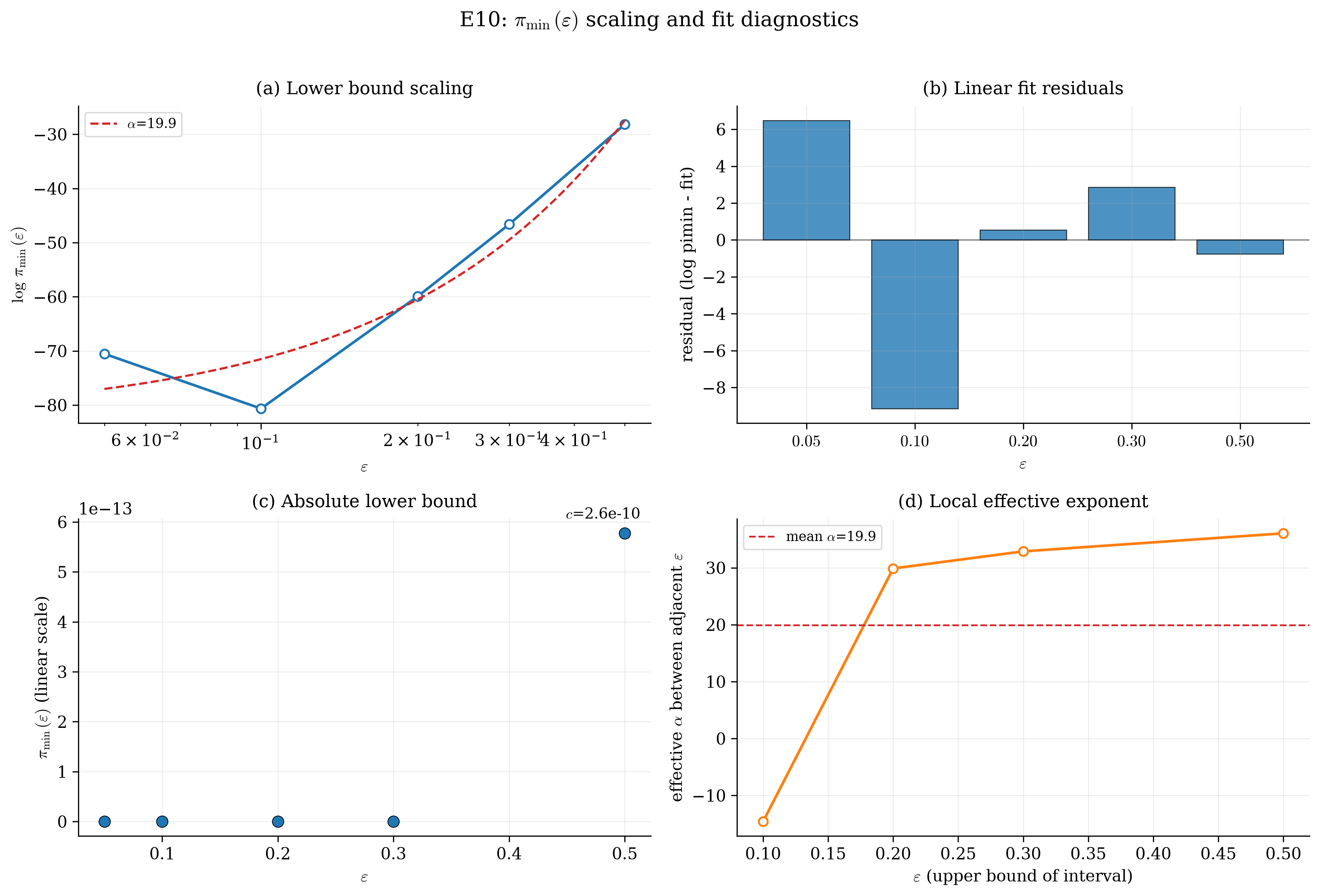}
\caption{The $\varepsilon$-scaling law of $\pi_{\min}$}
\label{fig:e10_pimin}
\end{figure}

\section{Conclusion}\label{sec:conclusion}

This paper establishes the statistical and algorithmic theory of IOT under a
feature-parameterized cost. The central technical contribution is the Sinkhorn
linearization and the spectral proxy (Section~3): the spectral sandwich of the restricted
Hessian yields a single core spectral bound, which, under the respective regularity
conditions, drives Theorem~\ref{thm:T1} (identifiability), Theorem~\ref{thm:T3}
(well-posedness), and Theorem~\ref{thm:T4} (population convergence with a conditional
finite-sample statement); Theorem~\ref{thm:T2} (sparsistency) additionally requires
irrepresentability of the actual OT information matrix, all-coordinate score concentration,
and a global selection condition; O5 (misspecification) interprets the estimation target as
the pseudo-true projection onto the OT model set. The main open problems: minimax lower
bounds for IOT estimation, global strong monotonicity without compactness, a proof of the
H\"older-continuity conjecture, and extensions to continuous state spaces. All numerical
verifications (E1--E10) are in Section~\ref{sec:exp}.

\FloatBarrier
\bibliographystyle{siamplain}
\bibliography{references}

\end{document}